\pdfoutput=1
\documentclass[11pt,a4paper]{article}
\usepackage[margin=2.5cm]{geometry}
\usepackage[T1]{fontenc}
\usepackage[utf8]{inputenc}
\usepackage{lmodern}
\usepackage{microtype}
\usepackage{amsmath,amssymb,amsfonts,amsthm,mathtools}
\usepackage{bbm}
\usepackage{booktabs}
\usepackage{multirow}
\usepackage{array}
\usepackage{graphicx}
\usepackage{xcolor}
\usepackage{enumitem}
\usepackage{tikz}
\usepackage{pgfplots}
\usetikzlibrary{arrows.meta,positioning,fit,calc}
\pgfplotsset{compat=1.18}
\usepgfplotslibrary{groupplots}
\usepackage{authblk}
\usepackage[numbers,square,sort&compress]{natbib}
\usepackage[labelfont=bf,labelsep=period,font=small]{caption}
\usepackage{titlesec}
\usepackage[hidelinks]{hyperref}

\titleformat{\section}{\Large\sffamily\bfseries}{}{0pt}{}
\titleformat{\subsection}{\large\sffamily\bfseries}{}{0pt}{}
\titleformat{\subsubsection}{\normalsize\sffamily\bfseries}{}{0pt}{}
\titlespacing*{\section}{0pt}{2.0ex plus 1ex minus .2ex}{1.0ex plus .2ex}
\titlespacing*{\subsection}{0pt}{1.6ex plus 1ex minus .2ex}{0.6ex plus .2ex}
\theoremstyle{plain}
\newtheorem{theorem}{Theorem}
\newtheorem{lemma}{Lemma}
\newtheorem{corollary}{Corollary}
\newtheorem{proposition}{Proposition}
\theoremstyle{definition}
\newtheorem{definition}{Definition}

\newcommand{\R}{\mathbb{R}}
\newcommand{\1}{\mathbbm{1}}
\newcommand{\E}{\mathbb{E}}
\newcommand{\TopK}{\mathrm{TopK}}
\newcommand{\softmax}{\mathrm{softmax}}

\newcommand{\gelu}{\mathrm{GELU}}
\newcommand{\LN}{\mathrm{LN}}
\newcommand{\Aff}{\mathrm{Aff}}
\newcommand{\mean}{\mathrm{mean}}
\newcommand{\Var}{\mathrm{Var}}
\newcommand{\eps}{\varepsilon}
\newcommand{\PW}{\mathrm{PW}}
\newcommand{\MAC}{\mathrm{MAC}}
\newcommand{\nrk}{\operatorname{nrk}}
\newcommand{\Bs}{\mathcal B_{\mathrm{strict}}}
\newcommand{\Be}{\mathcal B_{\mathrm{ext}}}

\newcommand{\Bh}{\mathcal B_{\odot}}
\newcommand{\Ttr}{\mathcal T_{\mathrm{Tr}}}
\newcommand{\Tkq}{\mathcal T_{k,q}}

\title{\sffamily\bfseries Transformers as Generic Deep Neural Networks: Exact Compilation, the Cost of Removing Activation--Activation Products, and Multi-Agent Attention Clustering}
\author[1,*]{Sergey Salishev}
\author[2]{Zeev Volkovich}
\author[1]{Oleg Granichin}
\affil[1]{St.~Petersburg State University, Saint Petersburg 199034, Russia}
\affil[2]{Department of Software Engineering, Braude College of Engineering, Karmiel 2161002, Israel}
\affil[*]{Corresponding author: s.salischev@spbu.ru}
\date{}

\begin{document}

\maketitle

\begin{abstract}
Transformers are deployed as a special model class. Neural processing units (NPU) add dedicated operators for attention and normalization, and quantization relies on Transformer-specific algorithms. We show the computation does not force this. Every Transformer block, with attention, LayerNorm and a GELU MLP, compiles exactly into a generic deep neural network built from fixed affine maps, routing and the square, sigmoid and inverse square root, plus GELU or its sigmoid approximation. Without activation--activation products, every transform-wise compilation adds at least $N^2$ scalar nonlinear evaluations per head for dense scores. Keeping that one operation type removes this cost; its elementwise form, present on every surveyed NPU, suffices with the sigmoid and inverse square root. For top-$k$ attention with exchangeable tail scores, we prove that exact or conditional-mean tail mass gives an unbiased, minimum-mean-square-error estimate of the attention distribution, while tail drop is conditionally biased. When queries or agents cluster, centroids choose the keys first, reducing the number of computed scores from $N^2$ to $Nk$ per head. On DeiT-Tiny, MambaIRv2-Light and Pythia-70M the compiled models match their references with all keys, standard post-training quantization applies, and the exact tail-mass estimate beats tail dropping with half as many keys on DeiT-Tiny and, for $k\le8$, on Pythia-70M.

\end{abstract}

\section{Introduction}

Transformers~\cite{vaswani2017attention} are treated as a model class apart from the deep neural networks (DNNs) that accelerators and tools were first built for. This view has two practical consequences.

The first is hardware. Accelerators built for convolutional networks run a fixed operator set and hand everything else to the host. The Edge TPU compiler cuts the model at the first unsupported operation and runs the rest on the CPU~\cite{coral_edgetpu_models}. NVIDIA's Jetson DLA has no matrix-multiply layer~\cite{nvidia_dla}. On phones, llm.npu runs the linear layers on the NPU and keeps LayerNorm and attention on the CPU or GPU~\cite{xu2025llmnpu}. Vendors answer with Transformer-specific operators and units for attention, softmax and normalization~\cite{park2020optimus,seo2024ianus,kim2024hwsoftmax}.

The second is quantization. Post-training quantization (PTQ) of Transformers relies on dedicated algorithms. Liu et al.\ search quantization scales with a ranking-aware objective~\cite{liu2021ptq4vit}, and PTQ4ViT adds twin-uniform quantizers for the post-softmax and post-GELU activations~\cite{yuan2022ptq4vit}. FQ-ViT adds power-of-two factors for LayerNorm and a log-integer softmax~\cite{lin2022fqvit}. RepQ-ViT reparameterizes quantization scales~\cite{li2023repqvit}, IGQ-ViT groups channels per instance~\cite{moon2024igqvit}, I-BERT replaces the nonlinearities by integer polynomials~\cite{kim2021ibert}, and SmoothQuant moves activation outliers into the weights of language models~\cite{xiao2023smoothquant}.

We show that the computation does not force this split. The argument has four steps.

\emph{Exact compilation.} Every Transformer block, with LayerNorm, scaled dot-product attention, softmax and a GELU MLP, compiles exactly into a generic DNN of fixed affine maps, scalar nonlinearities and routing, with no product of two activations (Theorem~\ref{thm:strict}). We call it the strict graph. Squares replace products, through $ab=\tfrac12[(a+b)^2-a^2-b^2]$ where the squares of the factors are shared and through polarization otherwise, and no exponential or logarithm is needed: squares, the sigmoid and the inverse square root suffice, with GELU or its approximation $x\,\sigma(1.702x)$ as the MLP activation.

\emph{The cost of strictness.} Since the strict graph has no dedicated layers, standard DNN tools apply to it. But removing products of two activations has a price. For free arguments, any compilation without them adds at least $N^2$ scalar nonlinear evaluations per head at $QK^\top$ and $Nd$ at $PV$, whatever the scalar functions, depth or width (Theorem~\ref{thm:lower}).

\emph{One more operation type.} Adding activation--activation products removes the price. Their elementwise form, the Hadamard product, is on every NPU we survey, and it suffices: with it, a block needs only the sigmoid and the inverse square root besides its MLP activation, and only softmax adds work, $2N^2$ evaluations per head (Proposition~\ref{prop:hsr}). With the exponential and the reciprocal it adds no pointwise work (Proposition~\ref{prop:hadamard}), and with matrix multiplication of two activations the extended compilation adds no work of either kind (Theorem~\ref{thm:ext}). On a GPU, for $k\le64$, the Hadamard and strict graphs run at $1.0$ to $1.3$ and $1.2$ to $1.5$ times the latency of this extended graph. Theorems~\ref{thm:lower} and~\ref{thm:ext} make this operation type necessary and sufficient for compilation without added pointwise work. Practice already uses this operation set. ONNX defines LayerNormalization and Softmax as functions over it~\cite{onnx_layernorm,onnx_softmax}, current NPUs expose it~\cite{seo2024ianus,amd_ryzenai,arm_ethosu_vela}, and convolutional networks with squeeze-and-excitation gating already use the elementwise product~\cite{hu2018senet,howard2019mobilenetv3}.

\emph{Clustered attention.} The attention weights of a query cluster on a few keys~\cite{gupta2021topk,wang2021spatten,gong2025hata} when its scores have a gap (Lemma~\ref{lem:gap}). Queries cluster too~\cite{geshkovski2025mathematical,vyas2020clustered}, and queries close to one point share their top keys (Lemma~\ref{lem:shared}). Top-$k$ attention keeps the $k$ largest scores of each query~\cite{gupta2021topk,wang2021spatten}, the topological rule of bird flocks~\cite{ballerini2008interaction}. We add an estimate of the omitted probability mass, spread uniformly over the dropped keys. For exchangeable tail scores, exact or conditional-mean tail mass gives the unbiased, minimum-mean-square-error attention-distribution estimate; tail drop is conditionally biased (Theorem~\ref{thm:tail}). Centroids of the query clusters~\cite{vyas2020clustered} choose the keys before the scores, so the strict graph forms only $Nk$ scores per head (Corollary~\ref{cor:sparse}, Proposition~\ref{prop:routed}). One step of the Motsch--Tadmor flocking model is one attention head (Proposition~\ref{prop:flock}).

We test the compilation on three model families: DeiT-Tiny~\cite{touvron2021deit} for image classification on Oxford-IIIT Pet~\cite{parkhi2012pet} and ImageNet-1k~\cite{deng2009imagenet}, MambaIRv2-Light~\cite{guo2025mambairv2} for super-resolution, and Pythia-70M~\cite{biderman2023pythia} for language modeling.

\emph{Related work.} Universal approximation gives some network within each error on a compact set~\cite{cybenko1989approximation,pinkus1999approximation}. Forming products from squares is classical: tables of quarter-squares multiply by polarization~\cite{laundy1856quarter}, Winograd's inner product forms the self-terms of each row once and halves the products of a matrix multiplication~\cite{winograd1968inner}, and expressivity proofs for ReLU networks use $ab=\tfrac12[(a+b)^2-a^2-b^2]$~\cite{yarotsky2017error}. Compilers such as Glow map a model onto the operators a device provides~\cite{rotem2018glow}. Privacy-preserving inference decomposes attention, softmax and LayerNorm into primitive operations for cryptographic protocols, usually with polynomial approximations~\cite{chen2022thex,hao2022iron,pang2024bolt}. Multiplication-free training replaces every product by an approximate piecewise-affine operation and trains on that arithmetic~\cite{kosson2023pam}. Training-time alternatives simplify LayerNorm~\cite{zhang2019rmsnorm} or remove it~\cite{bachlechner2021rezero,zhu2025dyt}. Top-$k$ attention saves memory on long contexts~\cite{gupta2021topk,wang2021spatten,gong2025hata}, usually by renormalizing the kept scores. Scatterbrain adds a low-rank estimate of the non-selected attention to a sparse one~\cite{chen2021scatterbrain}, Combiner attends to the rest of the sequence through region summaries~\cite{ren2021combiner}, and MagicPIG replaces top-$k$ selection by sampling to avoid its bias~\cite{chen2025magicpig}. Clustered attention chooses the kept keys by query centroids and gives the other keys the weights of the centroid~\cite{vyas2020clustered}; we use its key selection with calibrated centroids. Lookup tables, piecewise-linear units and integer approximations implement scalar nonlinearities in hardware~\cite{yu2022nnlut,upadhyay2024nova,kim2024hwsoftmax}.

\emph{What is new.} Unlike universal approximation, the compiled graph is exact on all inputs, reuses the trained weights and has a known size. Unlike compilers and privacy-preserving decompositions, we ask which small operator set expresses a whole block exactly. New are the lower bound that prices product-free compilation and the extended construction that removes the price. Unlike multiplication-free training, we compile a trained model, keep LayerNorm, and remove only products of two activations. Our tail-mass estimate is simpler than those of Scatterbrain, Combiner and clustered attention. It is a fixed affine map for a constant mass. For clustered queries, new are the cost bounds for the selected scores and the flocking step as one attention head.

\section{Results}
\subsection{Theory}
\label{sec:theory}

All proofs are in Methods.

\subsubsection{A generic DNN vocabulary}
\label{sec:vocab}

We describe a DNN by the types of operations its nodes perform. Three types cover plain convolutional networks and multilayer perceptrons.

\begin{definition}[Strict basis]
\label{def:strict}
Fix a set $\Phi$ of scalar functions. A graph over the strict basis $\Bs(\Phi)$ is a finite acyclic graph whose nodes are of three types:
\begin{enumerate}[leftmargin=1.5em,itemsep=0pt]
\item \emph{weight--activation contraction}: a fixed affine map $x\mapsto Ax+b$, where $A$ and $b$ are parameters of the graph and do not depend on the input;
\item \emph{unary map}: $x\mapsto\phi(x)$ applied entrywise, with $\phi\in\Phi$;
\item \emph{routing}: copy, reshape, permutation, concatenation, broadcast, masking, and selection of entries by comparisons of computed values (max, $\TopK$, gather), with a fixed tie-breaking rule.
\end{enumerate}
No node multiplies two quantities that both depend on the input; a square is a unary map. For every $\Phi$, we call $\Bs(\Phi)$ a \emph{product-free} basis and its graphs product-free graphs.
\end{definition}

Convolutions, fully connected layers, batch normalization at inference, residual sums and average pooling are fixed affine maps; max pooling is routing. ReLU, GELU and other activations are unary maps. Gating, as in squeeze-and-excitation blocks~\cite{hu2018senet}, multiplies two activations and is not in this basis. The set $\Phi$ is a parameter. Shifts such as $x+\eps$ and scalings such as $x/2$ belong to the preceding affine node, so a unary set needs only the shape of each function.

\begin{definition}[Strict unary set]
\label{def:phistrict}
$\Phi_{\mathrm{strict}}=\{(\cdot)^2,\ \sigma,\ (\cdot)^{-1/2},\ \gelu\}$, with $\sigma(t)=1/(1+e^{-t})$.
\end{definition}

Each function has one role. The square forms second moments and products: $ab=\tfrac12[(a+b)^2-a^2-b^2]$ where the squares of the factors are shared across products, as for the rows of $Q$, $K$ and $V$ and in LayerNorm, and polarization, $ab=\tfrac14[(a+b)^2-(a-b)^2]$, for an isolated product. The sigmoid gives the exponential, $e^t=\sigma(t)/\sigma(-t)$. The inverse square root normalizes LayerNorm and, squared, gives the reciprocal, $1/x=(x^{-1/2})^2$ for $x>0$. $\gelu$ is the MLP activation (Definition~\ref{def:transforms}). The set has no exponential and no logarithm.

\begin{definition}[Transform sets]
\label{def:transforms}
The Transformer transform set is
\[
\Ttr=\{\Aff,\ \LN,\ \mathrm{Score},\ \softmax,\ \mathrm{Mix},\ \mathrm{MLP}_{\gelu}\},
\]
where
\begin{align*}
\Aff(x)&=Ax+b,
&\LN(x)&=\beta+\gamma\odot\frac{x-\mean(x)}{\sqrt{\Var(x)+\eps}},\\
\mathrm{Score}(Q,K)&=QK^\top/\sqrt d,
&\mathrm{Mix}(P,V)&=PV,\\
\mathrm{MLP}_{\gelu}(x)&=A_2\gelu(A_1x+b_1)+b_2,
&\gelu(x)&=\tfrac x2\bigl(1+\operatorname{erf}(x/\sqrt2)\bigr),
\end{align*}
and $\softmax$ acts on each row of scores. A model may use the approximation $\gelu(x)\approx x\,\sigma(1.702x)$~\cite{hendrycks2016gelu} instead, which changes the model. Then $\gelu$ is no longer a basic operation: it is omitted from every unary set, and the sigmoid is added where missing. Here $\Var(x)=\mean_i[(x_i-\mean(x))^2]$, $\eps>0$ and $A,b,\beta,\gamma,A_1,b_1,A_2,b_2$ are fixed parameters, and one attention head has $N$ tokens, head dimension $d$, $Q,K,V\in\R^{N\times d}$ and $P\in\R^{N\times N}$.
\end{definition}

A compiled graph serves a fixed context capacity $N$, and the construction gives one graph for each $N$. Shorter inputs are padded and masked, and causal models use the causal mask. Masks are fixed routing and affine maps.

\subsubsection{Every Transformer block compiles exactly}
\label{sec:exact}

\begin{theorem}[Exact compilation into the strict basis]
\label{thm:strict}
Every finite Transformer block composed of transforms in $\Ttr$ is computed exactly by a finite graph over $\Bs(\Phi_{\mathrm{strict}})$.
\end{theorem}

We call the graph of Theorem~\ref{thm:strict} the \emph{strict graph}. A compilation is \emph{transform-wise} when it replaces each transform of the block by its own graph that computes that transform exactly; the graphs then compose. The strict graph is transform-wise: Lemmas~\ref{lem:basic}--\ref{lem:softmax} in Methods give the graph of each transform.

\subsubsection{Strict compilation pays an unavoidable pointwise cost}
\label{sec:cost}

A multiply--accumulate (MAC) array computes products of matrices and reuses each loaded number many times. A unary map takes one number and returns one, so it runs on a separate, slower scalar, vector or lookup-table unit. We count the two kinds of work separately.

\begin{definition}[Pointwise, contraction and added work]
\label{def:work}
For a graph $G$, the pointwise work $\PW(G)$ is the number of scalar evaluations of unary maps in one forward pass; it counts evaluations, not nodes. The contraction work $\MAC(G)$ is the number of multiply--accumulate steps in its fixed affine nodes and in its products of two activations. The \emph{source graph} of a transform $T$ is the graph of its defining formula in Definition~\ref{def:transforms}, with each product of two activations as one product node; $\PW(T)$ and $\MAC(T)$ are its work. $\mathrm{Score}$ is one product node with $\PW=0$ and $\MAC=N^2d$ per head, LayerNorm forms its second moment with a product node and pays one inverse square root per token, and softmax pays one exponential per entry and one reciprocal per row. If a transform-wise compilation replaces $T$ by the graph $G_T$, its added work is $\Delta\PW(T)=\PW(G_T)-\PW(T)$ and $\Delta\MAC(T)=\MAC(G_T)-\MAC(T)$.
\end{definition}

\begin{theorem}[Lower bound for product-free compilation]
\label{thm:lower}
Fix one attention head and any set $\Phi$ of scalar functions, and let the entries of $Q$, $K$, $P$ and $V$ range freely over $\R$. Every product-free graph that computes $\mathrm{Score}=QK^\top/\sqrt d$ exactly has $\PW\ge N^2$, and every product-free graph that computes $\mathrm{Mix}=PV$ exactly has $\PW\ge Nd$.
\end{theorem}

The source graphs of $\mathrm{Score}$ and $\mathrm{Mix}$ make no unary evaluation, so the bounds hold for the added work $\Delta\PW$. They hold for every depth, width and choice of affine coefficients, and they count evaluations, so fusing nodes does not lower them. On inputs where its routing is fixed, a product-free graph is an affine function of its unary outputs (Methods, Lemma~\ref{lem:affine}). It therefore needs at least as many unary evaluations as its target has coordinates that stay linearly independent modulo affine functions; for $\mathrm{Score}$ these are all $N^2$ entries.

The strict graph pays $N^2d+2Nd$ unary evaluations per head in $\mathrm{Score}$ and $N^2d+N^2+Nd$ in $\mathrm{Mix}$, a factor $d+2d/N$ above the bound of Theorem~\ref{thm:lower} for $\mathrm{Score}$ and $N+N/d+1$ above it for $\mathrm{Mix}$. Supplementary Section~S1 shows how far the bound of Theorem~\ref{thm:lower} can be raised.

For one DeiT-Tiny block with dense attention (Table~\ref{tab:pw}), the strict graph makes $61$ times the unary evaluations of the source block, and $92\%$ of them fall in $\mathrm{Score}$ and $\mathrm{Mix}$.

\begin{table}[t]
\caption{\textbf{Pointwise work per DeiT-Tiny block.} Unary evaluations per instance of each transform; $N=197$ tokens, $h=3$ heads of dimension $d=64$, MLP width $4hd$. Source: the source graph (Definition~\ref{def:work}), with GELU and dense attention. Bound: Theorem~\ref{thm:lower} for dense scores, Corollary~\ref{cor:sparse} for the $k$ routed scores. Strict: Theorem~\ref{thm:strict}; routed: Proposition~\ref{prop:routed}, with $k=16$ in the DeiT-Tiny rows. Hadamard: Proposition~\ref{prop:hsr}. Extended: Theorem~\ref{thm:ext}. Block totals count both LayerNorms.}
\label{tab:pw}
\centering
\footnotesize
\setlength{\tabcolsep}{3pt}
\begin{tabular}{l cccccc}
\toprule
Transform & Source & Bound, dense / routed & Strict & Strict, routed & Hadamard & Extended\\
\midrule
\multicolumn{7}{l}{\emph{Symbolic}}\\
$\LN$ & $N$ & & $2Nhd{+}2N$ & $2Nhd{+}2N$ & $N$ & $N$\\
$\softmax$ & $hN^2{+}hN$ & & $8hN^2{+}2hN$ & $8hNk{+}2hN$ & $3hN^2{+}hN$ & $hN^2{+}hN$\\
$\mathrm{MLP}_{\gelu}$ & $4Nhd$ & & $4Nhd$ & $4Nhd$ & $4Nhd$ & $4Nhd$\\
$\mathrm{Score}$ & $0$ & $hN^2$ / $hNk$ & $hN^2d{+}2hNd$ & $hNkd{+}2hNd$ & $0$ & $0$\\
$\mathrm{Mix}$ & $0$ & $hNd$ / $hNd$ & $hN^2d{+}hN^2{+}hNd$ & $hNkd{+}hNk{+}hNd$ & $0$ & $0$\\
\midrule
\multicolumn{7}{l}{\emph{DeiT-Tiny}}\\
$\LN$ & 197 & & $7.60\times10^{4}$ & $7.60\times10^{4}$ & 197 & 197\\
$\softmax$ & $1.17\times10^{5}$ & & $9.33\times10^{5}$ & $7.68\times10^{4}$ & $3.50\times10^{5}$ & $1.17\times10^{5}$\\
$\mathrm{MLP}_{\gelu}$ & $1.51\times10^{5}$ & & $1.51\times10^{5}$ & $1.51\times10^{5}$ & $1.51\times10^{5}$ & $1.51\times10^{5}$\\
$\mathrm{Score}$ & 0 & $1.16\times10^{5}$ / 9456 & $7.53\times10^{6}$ & $6.81\times10^{5}$ & 0 & 0\\
$\mathrm{Mix}$ & 0 & $3.78\times10^{4}$ / $3.78\times10^{4}$ & $7.61\times10^{6}$ & $6.52\times10^{5}$ & 0 & 0\\
\midrule
Block total & $2.69\times10^{5}$ & & $1.64\times10^{7}$ & $1.71\times10^{6}$ & $5.02\times10^{5}$ & $2.69\times10^{5}$\\
\bottomrule
\end{tabular}
\end{table}

\subsubsection{Activation--activation products remove the cost}
\label{sec:ext}

Theorem~\ref{thm:lower} names the missing operation type: the product of two activations. Every NPU we survey provides it as an elementwise product, together with the sigmoid (Supplementary Table~\ref{tab:s-npu}).

\begin{definition}[Hadamard basis]
\label{def:hadamard}
The Hadamard basis $\Bh(\Phi)$ has the three node types of $\Bs(\Phi)$ and a fourth, the \emph{Hadamard product} $a\odot b$ of two activations with broadcasting.
\end{definition}

\begin{definition}[Hadamard unary set]
\label{def:phih}
$\Phi_{\odot}=\{\sigma,\ (\cdot)^{-1/2},\ \gelu\}$.
\end{definition}

With the Hadamard product, the square is $x\odot x$, the exponential is $e^t=\sigma(t)\odot\bigl(\sigma(-t)^{-1/2}\odot\sigma(-t)^{-1/2}\bigr)$, and the reciprocal is $x^{-1/2}\odot x^{-1/2}$.

\begin{proposition}[Sigmoid and inverse square root suffice with the Hadamard product]
\label{prop:hsr}
Every finite Transformer block over $\Ttr$ is computed exactly by a finite graph over $\Bh(\Phi_{\odot})$ in which softmax adds $\Delta\PW=2N^2$ per head and every other transform adds $\Delta\PW=0$.
\end{proposition}

We call this graph the \emph{Hadamard graph}. Its products add $\Delta\MAC\le N^2d$ per head to $\mathrm{Score}$ and $\mathrm{Mix}$ and $2N^2+N$ to softmax. For one DeiT-Tiny block it makes $5.02\times10^{5}$ unary evaluations, against $1.64\times10^{7}$ for the strict graph (Table~\ref{tab:pw}).

The $2N^2$ evaluations of softmax come from building the exponential from sigmoids. With the exponential and the reciprocal as unary maps they vanish.

\begin{definition}[Extended unary set]
\label{def:phiext}
$\Phi_{\mathrm{ext}}=\{\exp,\ (\cdot)^{-1},\ (\cdot)^{-1/2},\ \gelu\}$.
\end{definition}

$\Phi_{\mathrm{ext}}$ is the set of scalar functions in the defining formulas of $\Ttr$: the exponential and the reciprocal of softmax, the inverse square root of LayerNorm and the MLP activation.

\begin{proposition}[The Hadamard product suffices]
\label{prop:hadamard}
Every finite Transformer block over $\Ttr$ is computed exactly by a finite graph over $\Bh(\Phi_{\mathrm{ext}})$ with $\Delta\PW(T)=0$ for every transform $T$. For $\mathrm{Score}$ and $\mathrm{Mix}$ a separate reduction node adds $\Delta\MAC\le N^2d$ per head; for every other transform $\Delta\MAC(T)=0$.
\end{proposition}

Matrix multiplication of two activations keeps the $N^2d$ products of $QK^\top$ and $PV$, which a Hadamard product writes to memory, inside the MAC array (Supplementary Section~S3).

\begin{definition}[Extended basis]
\label{def:ext}
The extended basis $\Be(\Phi)$ adds to $\Bh(\Phi)$ matrix multiplication (MatMul) and batched matrix multiplication (BMM) of two activations. With $\odot$ they form the \emph{activation--activation products}. For a common $\Phi$, $\Bs(\Phi)\subset\Bh(\Phi)\subset\Be(\Phi)$.
\end{definition}

\begin{theorem}[Exact compilation with no added work]
\label{thm:ext}
Every finite Transformer block over $\Ttr$ is computed exactly by a finite graph over $\Be(\Phi_{\mathrm{ext}})$, and this compilation has $\Delta\PW(T)=\Delta\MAC(T)=0$ for every transform $T$.
\end{theorem}

We call the graph of Theorem~\ref{thm:ext} the \emph{extended graph}. It is the source graph of each transform (Definition~\ref{def:work}), so it adds no work: $QK^\top$ and $PV$ stay one matrix-multiplication node each, with $N^2d$ multiply--accumulate steps per head and no unary evaluation. Theorems~\ref{thm:lower} and~\ref{thm:ext} together make activation--activation products necessary and sufficient for a transform-wise compilation that adds no pointwise work.

To turn the bound of Theorem~\ref{thm:lower} into time, let $R_{\mathrm{mac}}$ and $R_{\mathrm{pw}}$ be the throughputs of the MAC array and the pointwise unit, and $\kappa=R_{\mathrm{mac}}/R_{\mathrm{pw}}$.

\begin{corollary}[Time separation]
\label{cor:time}
Fix an $h$-head block. Assume that unary maps run only on the pointwise unit, at rate at most $R_{\mathrm{pw}}$, and that the extended graph runs $\mathrm{Score}$ and $\mathrm{Mix}$ on the MAC array at rate $R_{\mathrm{mac}}$. In every transform-wise product-free compilation, the $\mathrm{Score}$ and $\mathrm{Mix}$ subgraphs take at least $h(N^2+Nd)/R_{\mathrm{pw}}$; in the extended graph they take $2hN^2d/R_{\mathrm{mac}}$. The first time is longer whenever $\kappa>2Nd/(N+d)$, in particular whenever $\kappa\ge 2d$.
\end{corollary}

The strict graph makes more than $2hN^2d$ unary evaluations in $\mathrm{Score}$ and $\mathrm{Mix}$, so on them it is at least $\kappa$ times slower than the extended graph. For DeiT-Tiny ($N=197$, $d=64$) the threshold $2Nd/(N+d)$ is $96.6$. The extended basis maps onto a current NPU without a Transformer-specific unit; Figure~\ref{fig:overview} shows the unit that runs each node.

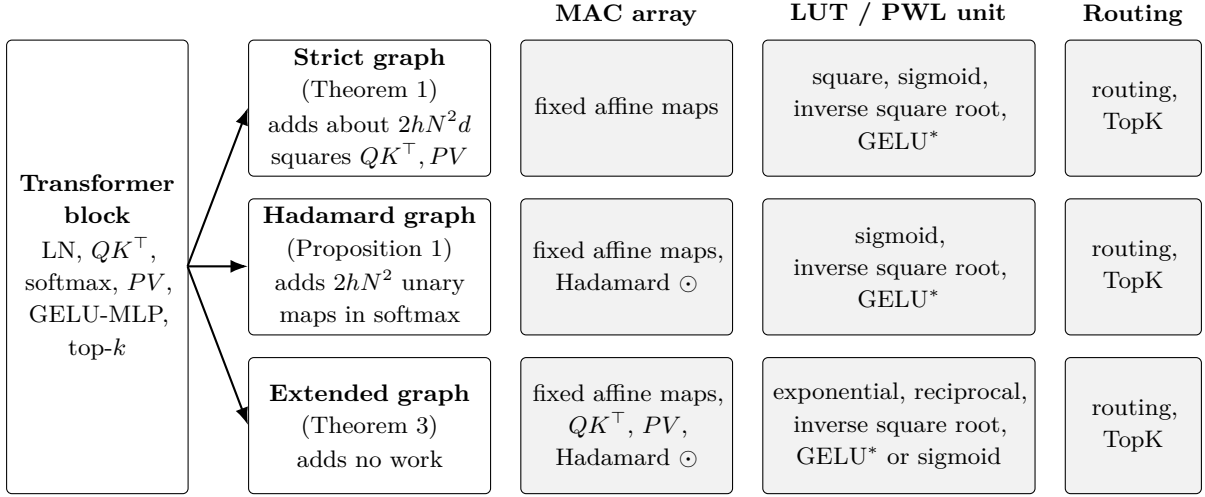
\begin{figure}[t]
\centering
\begin{tikzpicture}[
  font=\footnotesize,
  box/.style={draw, rounded corners=2pt, align=center, anchor=north west, inner xsep=0.1cm, inner ysep=0.1cm},
  gcell/.style={box},
  ucell/.style={box, fill=black!5},
  head/.style={font=\footnotesize\bfseries, anchor=south, align=center},
  arr/.style={-Latex, thick}
]
\node[gcell, text width=2.2cm, minimum height=6.0cm] (src) at (0,0) {\textbf{Transformer block}\\[1pt] LN, $QK^\top$,\\ softmax, $PV$,\\ GELU-MLP,\\ top-$k$};
\node[gcell, text width=3.0cm, minimum height=1.8cm] (s) at (3.2,0) {\textbf{Strict graph}\\ (Theorem~\ref{thm:strict})\\[1pt] adds about $2hN^2d$\\ squares $QK^\top, PV$};
\node[gcell, text width=3.0cm, minimum height=1.8cm] (hd) at (3.2,-2.1) {\textbf{Hadamard graph}\\ (Proposition~\ref{prop:hsr})\\[1pt] adds $2hN^2$ unary\\ maps in softmax};
\node[gcell, text width=3.0cm, minimum height=1.8cm] (e) at (3.2,-4.2) {\textbf{Extended graph}\\ (Theorem~\ref{thm:ext})\\[1pt] adds no work};
\node[head] at (8.2,0.05) {MAC array};
\node[head] at (11.8,0.05) {LUT / PWL unit};
\node[head] at (14.9,0.05) {Routing};
\node[ucell, text width=2.6cm, minimum height=1.8cm] at (6.8,0) {fixed affine maps};
\node[ucell, text width=3.4cm, minimum height=1.8cm] at (10.0,0) {square, sigmoid,\\ inverse square root,\\ GELU$^*$};
\node[ucell, text width=1.6cm, minimum height=1.8cm] at (14.0,0) {routing, $\TopK$};
\node[ucell, text width=2.6cm, minimum height=1.8cm] at (6.8,-2.1) {fixed affine maps,\\ Hadamard $\odot$};
\node[ucell, text width=3.4cm, minimum height=1.8cm] at (10.0,-2.1) {sigmoid,\\ inverse square root,\\ GELU$^*$};
\node[ucell, text width=1.6cm, minimum height=1.8cm] at (14.0,-2.1) {routing, $\TopK$};
\node[ucell, text width=2.6cm, minimum height=1.8cm] at (6.8,-4.2) {fixed affine maps,\\ $QK^\top$, $PV$,\\ Hadamard $\odot$};
\node[ucell, text width=3.4cm, minimum height=1.8cm] at (10.0,-4.2) {exponential, reciprocal,\\ inverse square root,\\ GELU$^*$  or sigmoid};
\node[ucell, text width=1.6cm, minimum height=1.8cm] at (14.0,-4.2) {routing, $\TopK$};
\draw[arr] (2.4,-3.0) -- (3.2,-0.9);
\draw[arr] (2.4,-3.0) -- (3.2,-3.0);
\draw[arr] (2.4,-3.0) -- (3.2,-5.1);
\end{tikzpicture}
\caption{\textbf{Three exact compilations of a Transformer block.} Each row lists the nodes of one graph under the NPU unit that runs them; $N$ tokens, $h$ heads of dimension $d$. $^*$if used in the source graph.}
\label{fig:overview}
\end{figure}

\subsubsection{Attention weights cluster on a few keys}
\label{sec:clustering}

Fix one query $q_i$, the keys $k_1,\dots,k_N$ and the values $v_1,\dots,v_N$, the rows of $K$ and $V$, and $1\le k<N$. Under a mask, $N$ counts the valid keys of the row. The scores are $s_j=q_i\cdot k_j/\sqrt d$, sorted as $s_{(1)}\ge\dots\ge s_{(N)}$, the attention distribution is $\pi=\softmax(s)$, and the attention output is $o=\sum_j\pi_jv_j$. Top-$k$ attention keeps the support $I=\TopK(s,k)$ and drops the other $m=N-k$ keys, the set $J$. The tail mass $\tau=\sum_{j\in J}\pi_j$ is the attention outside $I$, and $\Delta_k=s_{(k)}-s_{(k+1)}$ is the score gap. The attention weights cluster on $I$ when $\Delta_k$ is large. Queries cluster too.

\begin{lemma}[Score gap]
\label{lem:gap}
$\tau\le\frac mk\,e^{-\Delta_k}$.
\end{lemma}

\begin{lemma}[Shared support]
\label{lem:shared}
Let $c\in\R^d$ and $\rho\ge0$. If $\|q_i-c\|\le\rho$ and $\Delta_k>2\rho\max_j\|k_j\|/\sqrt d$, the scores $c\cdot k_j/\sqrt d$ have the same top $k$ keys $I$ as $q_i$.
\end{lemma}

By Lemma~\ref{lem:shared}, queries within $\rho$ of one point $c$ share their top $k$ keys when each meets the gap condition, and by Lemma~\ref{lem:gap} each of them puts at most $\frac mk e^{-\Delta_k}$ of its attention outside these keys.

\subsubsection{Top-$k$ attention with a tail-mass estimate}
\label{sec:topk}

The softmax over the kept scores, $\tilde\pi_j=\exp(s_j)/\sum_{l\in I}\exp(s_l)$ for $j\in I$, needs only $s_I$. The dense distribution on $I$ is $\pi_j=(1-\tau)\tilde\pi_j$. Write $a=\sum_{j\in I}\tilde\pi_jv_j$ for the mean of the kept values under $\tilde\pi$, $\bar u=\frac1m\sum_{j\in J}v_j$ for the plain mean of the dropped values, and $b=\frac1\tau\sum_{j\in J}\pi_jv_j$ for their mean under $\pi$; $\tau>0$ since every softmax weight is positive. Then $o=(1-\tau)a+\tau b$.

\begin{definition}[Tail-mass top-$k$ attention]
\label{def:tail}
For a tail mass $q\in[0,1]$, the estimate $\hat p$ of $\pi$ and the output $\hat o$ are
\[
\hat p_j=(1-q)\,\tilde\pi_j\ \ (j\in I),\qquad
\hat p_j=\frac{q}{m}\ \ (j\in J),\qquad
\hat o=(1-q)\,a+q\,\bar u .
\]
\emph{Calibrated} mode uses one constant $q$ per block, the mean tail mass of the reference model over calibration data. \emph{Exact} mode uses $q=\tau$. \emph{Tail drop}, the usual top-$k$ attention, uses $q=0$. For $k=0$, $I=\varnothing$, $m=N$ and $q=1$, so $\hat p_j=1/N$ and $\hat o$ is the mean of the values; calibrated and exact mode coincide. For $k\ge N$ the map is dense attention.
\end{definition}

The sparse transform set $\Tkq$ is $\Ttr$ with softmax and value mixing replaced by this map. This changes the model. Its source graph, used to count work, computes $\hat o$ in the all-values form
\begin{equation}
\label{eq:allvalues}
\hat o=\sum_{j\in I}\Bigl(\hat p_j-\frac qm\Bigr)v_j+\frac qm\sum_{j=1}^Nv_j ,
\end{equation}
which reads only the kept values and the sum of all values. Value mixing then costs $Nkd$ multiply--accumulate steps per head instead of $N^2d$, plus the sum of all values, kept as a prefix sum under a causal mask. Exact mode also forms all $N^2$ scores for the dense normalizer.

\begin{proposition}[Sparse attention compiles exactly]
\label{prop:sparse}
Every finite block over $\Tkq$ is computed exactly by a finite graph over $\Bs(\Phi_{\mathrm{strict}})$, and by a finite graph over $\Be(\Phi_{\mathrm{ext}})$ with $\Delta\PW(T)=\Delta\MAC(T)=0$ for every transform $T$.
\end{proposition}

Let $O=(I,s_I,V)$ be what top-$k$ attention reads: the kept keys, their scores and the values.

\begin{theorem}[Tail-mass estimate is unbiased and optimal]
\label{thm:tail}
Let the scores $s$ and the values $V$ be random, with $\E\|v_j\|^2<\infty$ for every $j$. Assume that the tail scores $(s_j)_{j\in J}$ are exchangeable given $O$: their conditional law is invariant under permutations of $J$.
\begin{enumerate}[label=(\roman*),leftmargin=2em,itemsep=0pt]
\item With $q=\E[\tau\mid O]$, the estimate satisfies $\hat p=\E[\pi\mid O]$ and $\hat o=\E[o\mid O]$. Both are unbiased, and they minimize $\E\|\hat p-\pi\|^2$ and $\E\|\hat o-o\|^2$ among all functions of $O$. In exact mode the same holds given $(O,\tau)$, with $q=\tau$.
\item Tail drop has conditional bias $\E[\hat p-\pi\mid O]=\E[\tau\mid O]\,(\tilde\pi_I,-\1_J/m)$ and $\E[\hat o-o\mid O]=\E[\tau\mid O]\,(a-\bar u)$. Its conditional distribution bias vanishes only when $\E[\tau\mid O]=0$.
\end{enumerate}
\end{theorem}

Theorem~\ref{thm:tail} concerns the attention distribution and the output of one head, not the loss of the whole model. Part (i) covers exact mode, and calibrated mode when $\tau$ is mean-independent of $O$. Among constant tail masses, $\E(q-\tau)^2=(q-\E\tau)^2+\Var\tau$ is minimal at $q=\E\tau$, which the calibrated $q$ estimates.

\begin{proposition}[Error decomposition]
\label{prop:error}
For every query,
\[
\hat o-o=\tau\,(\bar u-b)+(q-\tau)\,(\bar u-a),\qquad
\|\hat p-\pi\|_2^2=(q-\tau)^2\,w+\sum_{j\in J}\Bigl(\pi_j-\frac{\tau}{m}\Bigr)^2,
\]
with $w=\|\tilde\pi_I\|_2^2+\frac1m$. Tail drop has output error $\tau\,(a-b)$.
\end{proposition}

Proposition~\ref{prop:error} needs no assumption. The error splits into a uniform-tail term and a term from the error in the tail mass, and exact mode removes the second term, so for every query it minimizes the distribution error over $q$. When the tail is uniform, $\pi_j=\tau/m$ on $J$, so $b=\bar u$ and exact mode has no error, while tail drop keeps the error $\tau(a-b)$. In the expected distribution error, only $\E[(q-\tau)^2w]$ depends on $q$, so the best constant is $q^*=\E[w\tau]/\E[w]$, which equals $\E\tau$ when $w$ and $\tau$ are uncorrelated.

\subsubsection{Centroids factor the selection scores}
\label{sec:cluster}

Exact top-$k$ forms all $N^2$ scores before it selects, so its strict graph makes at least $N^2$ unary evaluations per head (Theorem~\ref{thm:lower}). When the queries cluster, the keys can be chosen before the scores, and only the $k$ kept scores of each query are formed.

\begin{corollary}[Cost of selected scores]
\label{cor:sparse}
Fix one head and $1\le k\le N$, and let each query select $k$ keys by a rule that is constant on some open set of inputs. Every product-free graph that computes the $Nk$ selected scores has $\PW\ge Nk$. Given the selected key indices, a graph over $\Bs(\Phi_{\mathrm{strict}})$ computes them with $\PW=Nkd+2Nd$.
\end{corollary}

For each head, k-means on calibration queries gives $C$ centroids $c_1,\dots,c_C$, the rows of $W\in\R^{C\times d}$. Query $i$ goes to its nearest centroid $c_{r(i)}$, the one that maximizes the affine score $q_i\cdot c_r-\tfrac12\|c_r\|^2$, and the assignment $E\in\{0,1\}^{N\times C}$ factors the queries as $Q\approx EW$. The selection scores $EWK^\top/\sqrt d$, a gather of the centroid scores $WK^\top/\sqrt d$, replace $QK^\top/\sqrt d$, and the routed support of query $i$ is the top $k$ keys of its row. It is the exact top-$k$ support of $q_i$ whenever the gap condition of Lemma~\ref{lem:shared} holds for $c=c_{r(i)}$. Like $\TopK$ of the scores, it is constant on an open set of inputs, so Corollary~\ref{cor:sparse} bounds its scores. The source graph of the routed map is the all-values form~\eqref{eq:allvalues} on the routed support.

\begin{proposition}[Routed support]
\label{prop:routed}
For fixed centroids and an $h$-head block, tail-mass attention with a constant tail mass on the routed supports is computed exactly by a graph over $\Bs(\Phi_{\mathrm{strict}})$ whose $\mathrm{Score}$ and $\mathrm{Mix}$ subgraphs make $h(2Nkd+Nk+3Nd)$ unary evaluations, and by a graph over $\Be(\Phi_{\mathrm{ext}})$ with $\Delta\PW=0$.
\end{proposition}

In both graphs the assignment scores and the centroid scores are fixed affine maps with $2hCNd$ multiply--accumulate steps, and the assignment and the selection are routing, so the router adds no unary evaluation. Proposition~\ref{prop:error} holds for every support, so the error of the routed estimate splits into the same uniform-tail and tail-mass terms.

\subsubsection{A flocking step is one attention head}
\label{sec:flock}

In the Motsch--Tadmor flocking model~\cite{motsch2011flocking}, agent $i$ at position $x_i$ with velocity $v_i$ follows
\[
\dot v_i=\alpha\Bigl(\sum_j\pi_{ij}v_j-v_i\Bigr),\qquad \pi_{ij}=\frac{g(\|x_i-x_j\|)}{\sum_lg(\|x_i-x_l\|)},
\]
where $j$ and $l$ run over all agents and $g$ is a decreasing influence function.

\begin{proposition}[Flocking step]
\label{prop:flock}
For $g(r)=e^{-\theta r^2}$ with $\theta>0$, $\pi_{ij}=\softmax_j(2\theta\,x_i^\top x_j-\theta\|x_j\|^2)$, so $\sum_j\pi_{ij}v_j$ is one head of distance-based attention~\cite{kim2021lipschitz} with queries $2\theta x_i$, keys $x_j$, a bias $-\theta\|x_j\|^2$ per key and values $v_j$. The top $k$ keys of agent $i$ are its $k$ nearest agents.
\end{proposition}

Tail-mass attention (Definition~\ref{def:tail}) gives the $k$ nearest agents of agent $i$ the weights $(1-q)\tilde\pi_j$ and each of the other $m$ agents the weight $q/m$. Proposition~\ref{prop:error} gives the error of this top-$k$ flocking step exactly, and Theorem~\ref{thm:tail} applies when the far positions are exchangeable given the kept scores, velocities and $x_i$. By the proof of Proposition~\ref{prop:flock}, Lemma~\ref{lem:shared} holds for these scores with $\sqrt d$ replaced by $1$, so agents close to one centroid share its routed support when the gap condition holds.

\subsection{Experiments}
\label{sec:experiments}

\subsubsection{The compiled models match their references}

We compile three pretrained models with the extended graph and top-$k$ tail-mass attention, and evaluate them before training. With all keys kept, the compiled models match their references: $91.36\%$ against $91.36\%$ on Pet, $72.01\%$ against $72.02\%$ on ImageNet, $32.51$\,dB against $32.51$\,dB on Set5, and a token negative log-likelihood (NLL) of $4.5496$ against $4.5496$ on WikiText-2. The strict and extended graphs compute the same model exactly (Theorems~\ref{thm:strict} and~\ref{thm:ext}). Only model changes, such as top-$k$ attention, and rounding change accuracy.

\subsubsection{The strict graph pays the predicted cost}

Table~\ref{tab:strict-fast} measures the cost of the strict graphs of Theorem~\ref{thm:strict} and Proposition~\ref{prop:sparse} on a GPU. The reference calls fused attention and LayerNorm kernels. The extended graph runs at $2.5$ to $5.6$ times the reference latency in Tables~\ref{tab:strict-fast} and~\ref{tab:deit}. The strict graph takes $1.2$ to $2.4$ times the latency and, at top-$64$, $2.3$ times the peak memory of the extended graph. The GPU has $\kappa\approx2$ for squares (Methods), far below the threshold of Corollary~\ref{cor:time}, so this slowdown does not come from the bound of Theorem~\ref{thm:lower}. It comes from the construction: $d$ times more pointwise work than the bound, and $N^2d$ operands per head that the implementation writes to memory. A per-module profile puts $73$ to $77\%$ of the added time in $\mathrm{Score}$ and $\mathrm{Mix}$ (Supplementary Table~\ref{tab:s-profile}). A merged GPU kernel avoids these writes; no transform-wise product-free graph avoids the bound of Theorem~\ref{thm:lower}.

With the same GELU approximation, the strict graph stays within $0.6$ points of the Hadamard graph and of its polarization variant with top-$k$ support, and within $0.4$ points routed. Two runs of one strict configuration differ by up to $0.5$ points with top-$k$ support and $1.5$ points routed, since the tail mass and centroids are calibrated on augmented training images. Squares cost no accuracy beyond run-to-run variation. With top-$k$ support, polarization makes $1.9$ times the unary evaluations and, for $k\le64$, runs $13$ to $14$ times slower with $5.5$ times the memory (Supplementary Section~S5). In integer arithmetic both identities are exact (Methods).

\begin{table}[t]
\caption{\textbf{Measured cost of the compilations.} DeiT-Tiny on Pet, bf16, before distillation. Accuracy (\%), and latency and peak GPU memory relative to the reference, which takes $2.20$\,ms per batch of 32 and $62$\,MB (Methods). Hadamard: Proposition~\ref{prop:hsr}. Strict: Theorem~\ref{thm:strict}; polarization: polarization for every product. The extended graph uses GELU, the others $x\odot\sigma(1.702x)$. Top-$k$: the $k$ largest exact scores. Routed: $k$ keys chosen by 32 query centroids before the scores (Methods). At $k=197$ both supports hold every key.}
\label{tab:strict-fast}
\centering
\small
\setlength{\tabcolsep}{2.4pt}
\begin{tabular}{r ccc ccc ccc ccc}
\toprule
& \multicolumn{3}{c}{Extended} & \multicolumn{3}{c}{Hadamard} & \multicolumn{3}{c}{Strict} & \multicolumn{3}{c}{Strict, polarization}\\
\cmidrule(lr){2-4}\cmidrule(lr){5-7}\cmidrule(lr){8-10}\cmidrule(lr){11-13}
$k$ & Acc. & Latency & Memory & Acc. & Latency & Memory & Acc. & Latency & Memory & Acc. & Latency & Memory\\
\midrule
\multicolumn{13}{l}{\emph{Top-$k$ support}}\\
16  & 86.24 & $3.28\times$ & $1.62\times$ & 83.51 & $4.18\times$ & $8.72\times$ & 84.08 & $5.04\times$ & $8.65\times$ & 83.95 & $68.3\times$ & $47.1\times$\\
32  & 89.23 & $3.89\times$ & $2.32\times$ & 87.46 & $4.73\times$ & $8.72\times$ & 87.98 & $5.37\times$ & $8.65\times$ & 87.71 & $76.1\times$ & $47.1\times$\\
64  & 90.35 & $4.88\times$ & $3.70\times$ & 89.59 & $5.75\times$ & $8.72\times$ & 89.78 & $7.14\times$ & $8.65\times$ & 89.70 & $93.1\times$ & $47.1\times$\\
197 & 91.36 & $5.41\times$ & $2.45\times$ & 90.35 & $10.5\times$ & $17.3\times$ & 90.73 & $13.0\times$ & $17.2\times$ & 90.57 & $132\times$ & $63.1\times$\\
\midrule
\multicolumn{13}{l}{\emph{Routed support}}\\
16  & 80.13 & $3.35\times$ & $1.54\times$ & 78.60 & $3.48\times$ & $2.21\times$ & 78.22 & $4.73\times$ & $2.19\times$ & 78.44 & $14.1\times$ & $5.90\times$\\
32  & 87.14 & $4.03\times$ & $2.28\times$ & 85.66 & $4.10\times$ & $3.56\times$ & 85.96 & $4.73\times$ & $3.55\times$ & 85.55 & $22.3\times$ & $11.0\times$\\
64  & 89.81 & $4.37\times$ & $3.73\times$ & 88.63 & $5.56\times$ & $6.27\times$ & 88.88 & $6.67\times$ & $6.23\times$ & 88.99 & $40.1\times$ & $21.2\times$\\
\bottomrule
\end{tabular}
\end{table}

\subsubsection{The elementwise product removes most of the strict cost}

With top-$k$ support and $k\le64$, the Hadamard graph of Proposition~\ref{prop:hsr} runs at $1.2$ to $1.3$ times the latency of the extended graph, against $1.4$ to $1.5$ for the strict graph, at similar memory (Table~\ref{tab:strict-fast}). Its measured unary counts equal Table~\ref{tab:pw}: $2.2\times10^6$ per image at $k=16$, against $1.1\times10^8$ for the strict graph. Its remaining cost is the $N^2d$ products of $QK^\top$ per head, which this implementation writes to memory (Supplementary Table~\ref{tab:s-profile}). With top-$k$ support, the approximation of GELU costs the Hadamard graph $0.7$ to $2.6$ and the strict graph $0.3$ to $2.3$ points; with GELU, the Hadamard graph stays within $0.1$ points of the extended graph.

\subsubsection{NPUs provide the Hadamard basis}

Supplementary Table~\ref{tab:s-npu} shows where each graph runs on ten edge accelerators by documented operator support, with the approximation of GELU where GELU is missing. The R-Car V4H runs both graphs with its DSP. Of the other nine, the strict graph runs entirely on seven, including the Edge TPU and, in INT8, the Jetson DLA; the open-source NVDLA IP and Ethos-N78 lack the inverse square root. The extended graph runs entirely on five; the toolchains of the other four have no matrix product of two activations and place it on the host. All ten multiply two activations elementwise, so none is product-free. The Hadamard graph uses the operations of the strict graph, so it runs wherever the strict graph runs.

\subsubsection{Standard quantization applies to the compiled graph}

We apply the post-training quantization (PTQ) used for convolutional networks~\cite{jacob2018quantization,nagel2021white}: symmetric max-abs scales from calibration data, with no training or reconstruction and no Transformer-specific quantizer. The only fitted quantities are least-squares output maps (Methods).

The strict graph in Table~\ref{tab:ptq} replaces GELU by $x\,\sigma(1.702x)$, so it uses the reduced strict basis $\Bs(\{(\cdot)^2,\sigma,(\cdot)^{-1/2}\})$. The extended graph keeps GELU and the matrix products of $\Be(\Phi_{\mathrm{ext}})$. Quantizers cut each graph into units with integer inputs and outputs (Methods). Weights, the logits and the inputs of linear, convolution and product nodes take 8 bits; the fixed affine nodes of the strict graph take 12-bit inputs, since these include squares and sums of squares.

PTQ changes the accuracy of the strict graph by $-1.47$ to $-0.05$ points with per-channel weight scales, and by $-1.06$ to $-0.14$ with one scale per weight tensor. On the extended graph it changes accuracy by $-2.5$ to $-3.9$ points.

\begin{table}[t]
\caption{\textbf{Standard post-training quantization of the compiled DeiT-Tiny on Pet.} Top-1 accuracy (\%) of the bf16 float model and after PTQ of the whole model with per-channel or per-tensor weight scales. On the calibration data the fp32 simulation holds every integer sum exactly (Methods). Nodes: quantized nodes, the output identity included. PV: $k/N$. Strict: the reduced strict graph, before distillation. Extended: the Jeffreys-distilled extended graph.}
\label{tab:ptq}
\centering
\setlength{\tabcolsep}{4.5pt}
\begin{tabular}{rcl r r rr rr}
\toprule
& & & & & \multicolumn{2}{c}{Per-channel} & \multicolumn{2}{c}{Per-tensor}\\
\cmidrule(lr){6-7}\cmidrule(lr){8-9}
$k$ & PV & Graph & Nodes & bf16 & PTQ & $\Delta$ & PTQ & $\Delta$\\
\midrule
16 & 0.08 & strict & 380 & 84.08 & 82.61 & $-1.47$ & 83.02 & $-1.06$\\
32 & 0.16 & strict & 380 & 87.98 & 87.30 & $-0.68$ & 87.27 & $-0.71$\\
64 & 0.32 & strict & 380 & 89.78 & 89.72 & $-0.05$ & 89.64 & $-0.14$\\
197 & 1.00 & strict & 368 & 90.73 & 90.27 & $-0.46$ & 89.75 & $-0.98$\\
\midrule
16 & 0.08 & extended & 247 & 88.58 & 86.13 & $-2.45$ & 84.90 & $-3.68$\\
32 & 0.16 & extended & 247 & 90.19 & 86.26 & $-3.92$ & 86.54 & $-3.65$\\
64 & 0.32 & extended & 247 & 90.76 & 87.44 & $-3.33$ & 87.16 & $-3.60$\\
197 & 1.00 & extended & 235 & 91.36 & 88.74 & $-2.62$ & 88.50 & $-2.86$\\
\bottomrule
\end{tabular}
\end{table}

\subsubsection{The tail-mass estimate beats tail drop}

\emph{Image classification.} Table~\ref{tab:deit} reports Oxford-IIIT Pet and ImageNet-1k with calibrated tail mass. The Jeffreys divergence $D_J$ between the predictive distributions of the compiled and reference models falls from $0.98$ at $k=8$ to $0.043$ at $k=64$ and $2\times10^{-4}$ with all keys. One epoch of distillation~\cite{hinton2015distilling} against the reference with the Jeffreys divergence~\cite{jeffreys1946invariant}, with $q$ trainable, brings top-$64$ to $90.43\pm0.18\%$ over five seeds. On ImageNet, without training, top-$64$ is within $1.19$ points of the reference.

\begin{table}[t]
\caption{\textbf{DeiT-Tiny compiled with calibrated tail-mass top-$k$.} Pet: top-1 accuracy (\%) of the compiled model, Jeffreys divergence $D_J$ to the reference, accuracy after one epoch of Jeffreys distillation (mean $\pm$ s.d. over five seeds), and latency of the compiled forward pass relative to the reference (Methods). ImageNet-1k: accuracy of the compiled model without training. PV is $k/N$. References: $91.36\%$ (Pet), $72.02\%$ (ImageNet).}
\label{tab:deit}
\centering
\begin{tabular}{r c ccccc c}
\toprule
& & \multicolumn{4}{c}{Oxford-IIIT Pet} & & ImageNet-1k\\
\cmidrule(lr){3-6}\cmidrule(lr){8-8}
$k$ & PV & Compiled & $D_J$ & Distilled & Latency & & Compiled\\
\midrule
1   & 0.01 & 15.73 & 6.26 & $41.22\pm0.60$ & $2.45\times$ & & 10.84\\
2   & 0.01 & 42.95 & 4.07 & $65.30\pm0.47$ & $2.85\times$ & & 27.79\\
4   & 0.02 & 67.10 & 2.20 & $78.22\pm0.18$ & $2.99\times$ & & 45.03\\
8   & 0.04 & 79.50 & 0.97 & $84.91\pm0.15$ & $3.06\times$ & & 57.20\\
16  & 0.08 & 86.24 & 0.38 & $88.56\pm0.19$ & $3.28\times$ & & 64.66\\
32  & 0.16 & 89.23 & 0.14 & $90.18\pm0.08$ & $3.89\times$ & & 68.81\\
64  & 0.32 & 90.35 & 0.045 & $90.43\pm0.18$ & $4.88\times$ & & 70.83\\
96  & 0.49 & 90.76 & 0.016 & $90.86\pm0.22$ & $5.32\times$ & & 71.50\\
128 & 0.65 & 91.11 & 0.0055 & $91.11\pm0.00$ & $5.61\times$ & & 71.82\\
197 & 1.00 & 91.36 & 0.0002 & $91.36\pm0.00$ & $5.41\times$ & & 72.01\\
\bottomrule
\end{tabular}
\end{table}

\emph{Image restoration.} MambaIRv2-Light~\cite{guo2025mambairv2} is a super-resolution model that combines state-space scans with windowed attention over $N=256$ positions. Its attention and LayerNorm modules compile as in DeiT-Tiny, and the state-space scans stay as they are. Table~\ref{tab:mambair} reports $\times4$ and $\times2$ super-resolution on Set5~\cite{bevilacqua2012set5}, before and after three epochs of $L_1$ fine-tuning on DIV2K~\cite{agustsson2017div2k}. Before training, exact tail mass gives the highest PSNR at every $k<N$, and top-$64$ is within $0.03$\,dB of the reference. After fine-tuning, exact mode and tail drop end within $0.06$\,dB of each other for $k\ge4$.

\begin{table}[t]
\caption{\textbf{MambaIRv2-Light on Set5.} PSNR (dB) before (pre) and after (post) three epochs of $L_1$ fine-tuning on DIV2K, selected on DIV2K training crops (Methods). References: $32.51$\,dB ($\times4$), $38.24$\,dB ($\times2$). At $k=0$ all weight sits on the uniform tail, which tail drop cannot express; at $k=256$ the modes coincide and the tail-drop cells repeat the calibrated run.}
\label{tab:mambair}
\centering
\begin{tabular}{r cc cc cc cc}
\toprule
& \multicolumn{6}{c}{$\times4$} & \multicolumn{2}{c}{$\times2$}\\
\cmidrule(lr){2-7}\cmidrule(lr){8-9}
& \multicolumn{2}{c}{Tail drop} & \multicolumn{2}{c}{Calibrated} & \multicolumn{2}{c}{Exact} & \multicolumn{2}{c}{Calibrated}\\
\cmidrule(lr){2-3}\cmidrule(lr){4-5}\cmidrule(lr){6-7}\cmidrule(lr){8-9}
$k$ & pre & post & pre & post & pre & post & pre & post\\
\midrule
0 & -- & -- & 30.22 & 31.10 & 30.22 & 31.10 & 35.98 & 37.24\\
1 & 22.11 & 30.64 & 30.45 & 31.32 & 30.95 & 31.73 & 36.17 & 37.36\\
2 & 26.05 & 31.63 & 30.62 & 31.55 & 31.36 & 31.97 & 36.35 & 37.52\\
4 & 29.66 & 32.07 & 30.85 & 31.81 & 31.71 & 32.11 & 36.50 & 37.62\\
16 & 31.99 & 32.37 & 31.55 & 32.23 & 32.19 & 32.34 & 37.05 & 37.90\\
64 & 32.43 & 32.47 & 32.27 & 32.45 & 32.48 & 32.51 & 37.90 & 38.13\\
256 & 32.51 & 32.51 & 32.51 & 32.51 & 32.50 & 32.52 & 38.23 & 38.23\\
\bottomrule
\end{tabular}
\end{table}

\emph{Language modeling.} Pythia-70M~\cite{biderman2023pythia} is a decoder-only model with a causal mask, evaluated on WikiText-2~\cite{merity2017wikitext} with $N=128$ tokens. Causal rows have different numbers of valid keys, so a single calibrated $q$ does not apply, and we use exact tail mass. Table~\ref{tab:pythia} shows that exact tail mass keeps the NLL within $0.01$ nats of the reference down to $k=16$, while tail drop loses $0.04$ nats there.

\begin{table}[t]
\caption{\textbf{Pythia-70M on WikiText-2.} Token NLL (nats) and perplexity (PPL) of the compiled model without training, with exact tail mass and with tail drop, over $2.53\times10^{5}$ validation tokens in float32. $\bar\tau$ is the mean omitted tail mass of top-$k$ on the dense scores of 4 random training minibatches of 4 sequences. Reference: NLL $4.5496$, PPL $94.59$.}
\label{tab:pythia}
\centering
\begin{tabular}{r ccc ccc c}
\toprule
& \multicolumn{3}{c}{Exact tail mass} & \multicolumn{3}{c}{Tail drop} & \\
\cmidrule(lr){2-4}\cmidrule(lr){5-7}
$k$ & NLL & $\Delta$NLL & PPL & NLL & $\Delta$NLL & PPL & $\bar\tau$\\
\midrule
128 & 4.5496 & 0.0000 & 94.59 & 4.5496 & 0.0000 & 94.59 & 0.000\\
64  & 4.5497 & 0.0002 & 94.61 & 4.5497 & 0.0002 & 94.61 & 0.004\\
32  & 4.5513 & 0.0017 & 94.75 & 4.5567 & 0.0072 & 95.27 & 0.023\\
16  & 4.5584 & 0.0088 & 95.43 & 4.5885 & 0.0390 & 98.35 & 0.062\\
8   & 4.5822 & 0.0326 & 97.73 & 4.6811 & 0.1316 & 107.89 & 0.123\\
4   & 4.6534 & 0.1038 & 104.94 & 4.9232 & 0.3737 & 137.44 & 0.205\\
2   & 4.8267 & 0.2771 & 124.80 & 5.5418 & 0.9923 & 255.14 & 0.313\\
1   & 5.1804 & 0.6308 & 177.76 & 6.6381 & 2.0885 & 763.65 & 0.440\\
\bottomrule
\end{tabular}
\end{table}

Figure~\ref{fig:tail} compares the three modes. On DeiT-Tiny/Pet, exact mode at $k$ beats tail drop at $2k$ for all tested $k\le32$; calibrated mode does so for $k\le16$. At equal $k=8$, calibrated mode gains $73.8$ points over tail drop. The per-head optimality of Theorem~\ref{thm:tail} does not order model accuracy: for $k\le4$ calibrated mode beats exact mode. On Pythia-70M, exact mode at $k$ beats tail drop at $2k$ for every $k\le8$. On MambaIRv2-Light, exact mode is the best of the three at every $k<N$. Calibrated mode falls below tail drop there for $k\ge16$, since one constant per block fits poorly a per-query tail mass with s.d.\ about $0.2$ at a mean of $0.57$ ($k=16$).

\begin{figure}[t]
\centering
\begin{tikzpicture}
\pgfplotsset{
  dropstyle/.style={color=black!55, mark=triangle*, mark size=2pt, thick},
  calstyle/.style={color=blue!75!black, mark=*, mark size=1.6pt, thick},
  exactstyle/.style={color=red!75!black, mark=square*, mark size=1.6pt, thick},
}
\begin{groupplot}[
  group style={group size=3 by 1, horizontal sep=1.55cm},
  width=5.1cm, height=4.6cm,
  xmode=log, log basis x=2,
  xlabel={$k$},
  tick label style={font=\scriptsize},
  label style={font=\small},
  title style={font=\small},
  grid=major, grid style={black!10},
]
\nextgroupplot[title={(a) DeiT-Tiny, Pet}, ylabel={accuracy (\%)}, ymin=0, ymax=95,
  xtick={1,4,16,64,197}, xticklabels={1,4,16,64,197},
  legend to name=tailleg, legend columns=4, legend style={font=\small, draw=none, column sep=8pt}]
\addplot[dropstyle] coordinates {(1,2.53) (2,3.16) (4,2.83) (8,5.70) (16,39.44) (32,81.71) (64,89.10)};
\addlegendentry{tail drop}
\addplot[calstyle] coordinates {(1,15.73) (2,42.95) (4,67.10) (8,79.50) (16,86.24) (32,89.23) (64,90.35) (96,90.76) (128,91.11) (197,91.36)};
\addlegendentry{calibrated}
\addplot[exactstyle] coordinates {(1,8.69) (2,21.31) (4,58.60) (8,81.77) (16,88.80) (32,90.84) (64,91.36) (96,91.39) (128,91.28) (197,91.36)};
\addlegendentry{exact}
\addplot[dashed, black, domain=1:197, samples=2] {91.36};
\addlegendentry{reference}
\nextgroupplot[title={(b) Pythia-70M, WikiText-2}, ylabel={$\Delta$NLL (nats)}, ymode=log, ymin=1e-4, ymax=4,
  xtick={1,4,16,64}, xticklabels={1,4,16,64}]
\addplot[dropstyle] coordinates {(1,2.0885) (2,0.9923) (4,0.3737) (8,0.1316) (16,0.0390) (32,0.0072) (64,0.0002)};
\addplot[exactstyle] coordinates {(1,0.6308) (2,0.2771) (4,0.1038) (8,0.0326) (16,0.0088) (32,0.0017) (64,0.0002)};
\nextgroupplot[title={(c) MambaIRv2-Light $\times4$, Set5}, ylabel={PSNR (dB)}, ymin=21.5, ymax=33.2,
  xtick={1,4,16,64,256}, xticklabels={1,4,16,64,256}]
\addplot[dropstyle] coordinates {(1,22.111) (2,26.047) (4,29.661) (16,31.987) (64,32.429)};
\addplot[calstyle] coordinates {(1,30.448) (2,30.624) (4,30.851) (16,31.549) (64,32.275) (256,32.505)};
\addplot[exactstyle] coordinates {(1,30.951) (2,31.359) (4,31.705) (16,32.192) (64,32.483) (256,32.498)};
\addplot[dashed, black, domain=1:256, samples=2] {32.51};
\end{groupplot}
\node[anchor=north] at ($(group c2r1.south)+(0,-1.0cm)$) {\pgfplotslegendfromname{tailleg}};
\end{tikzpicture}
\caption{\textbf{Tail-mass estimate against tail drop.} Quality of the compiled model before any training, as a function of the number of kept keys $k$. (a) DeiT-Tiny on Oxford-IIIT Pet, top-1 accuracy. (b) Pythia-70M on WikiText-2, increase of the token negative log-likelihood over the reference ($N=128$); calibrated mode does not apply to causal rows, whose numbers of valid keys differ. (c) MambaIRv2-Light $\times4$ super-resolution on Set5, PSNR.}
\label{fig:tail}
\end{figure}
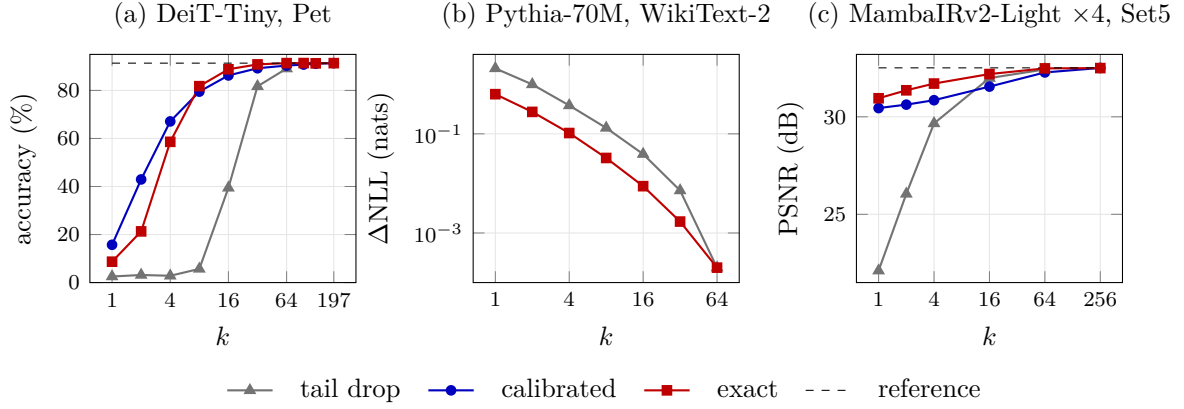

\subsubsection{Clustered queries cut the cost of the strict graph}

The measured unary counts equal Proposition~\ref{prop:routed} at every tested $k$. At $k=16$ the routed strict DeiT-Tiny graph uses $2.2\times$ instead of $8.6\times$ the reference memory and makes $2.4\times10^7$ instead of $1.1\times10^8$ unary evaluations per image (Table~\ref{tab:strict-fast}), at $78.22\%$ against $84.08\%$ accuracy; its latency falls from $5.0$ to $4.7$ times the reference. With twice the keys it matches exact top-$k$: at $k=32$ it reaches $85.96\%$ with $37\%$ of the unary evaluations and $41\%$ of the memory of exact top-$16$.

The routed support is accurate where the queries of a cluster attend to the same keys. On DeiT-Tiny it holds $83\%$ to $95\%$ of the attention mass of exact top-$k$ for $k=8$ to $64$, and $99\%$ in the last block, where the queries are most clustered. Pythia-70M queries cluster more tightly, but queries of one cluster attend to different keys, and with tail drop its perplexity rises by $9\%$ at $k=32$. Routing leaves the extended graph unchanged at $N=197$ and cuts the Hadamard graph's memory by $75\%$ at $k=16$.

\section{Discussion}

An NPU needs no attention, softmax or normalization unit to run Transformers. The elementwise product of two activations removes the pointwise cost that Theorem~\ref{thm:lower} prices, and a MAC array that accepts two activation operands only keeps the $N^2d$ products of $QK^\top$ and $PV$ out of memory.

Theorem~\ref{thm:lower} applies to transform-wise compilation, which keeps the attention contractions as subgraphs; a compilation that never forms $QK^\top$ is outside it (Supplementary Section~S1). Corollary~\ref{cor:time} needs hardware whose MAC array outruns its pointwise unit by more than $2Nd/(N+d)$.

The construction extends to other blocks. RMSNorm~\cite{zhang2019rmsnorm} compiles as LayerNorm without the mean. A gated MLP such as SwiGLU~\cite{shazeer2020glu} multiplies two activations, and its SiLU is $x\,\sigma(x)$, so both take a polarization graph in the strict basis and a Hadamard node in the Hadamard and extended bases.

For autonomous and multi-agent systems, the Hadamard graph lets a robot or drone run a Transformer policy on an NPU built for convolutional networks. By Proposition~\ref{prop:flock}, top-$k$ attention is the topological rule of starling flocks, where each bird follows a fixed number of nearest neighbors~\cite{ballerini2008interaction}, and the tail-mass term describes the far agents by their total weight and mean velocity, as cluster-level models of multi-agent control do~\cite{erofeeva2024multilevel}.

\emph{Future work.} The latency and memory reported here are GPU measurements, and the NPU placement of Supplementary Table~\ref{tab:s-npu} comes from documented operator support. Compiling the three graphs and their quantized integer units with the NPU toolchains, and running them on the devices, is the next step. A per-query predictor of the tail mass from the selected scores would bring calibrated mode toward exact mode without the dense normalizer. A swarm optimizer built on top-$k$ attention with a tail-mass term is the multi-agent counterpart of these results.

\section{Methods}

\subsection{Notation}

$\langle\mathcal B\rangle$ is the class of finite acyclic graphs over a basis $\mathcal B$. It is closed under composition. We write $F\in\langle\mathcal B\rangle$ when a graph in $\langle\mathcal B\rangle$ computes $F$ exactly. $M(x,y)$ is the polarization graph of $xy$: one fixed affine node forms $x\pm y$, two squares follow, and one fixed affine node returns $\tfrac14[(x+y)^2-(x-y)^2]=xy$. $\tilde M(x,y)$ is the norm-expansion graph of $xy$ when $x^2$ and $y^2$ are available: one fixed affine node forms $x+y$, one square follows, and one fixed affine node returns $\tfrac12[((x+y)^2-x^2)-y^2]=xy$, subtracting the two close squares first. For an index tuple $I$, $\operatorname{gather}(s,I)=(s_{i_1},\dots,s_{i_k})$. $\TopK$ returns indices in descending order of score with a fixed tie-breaking rule.

\subsection{Compilation lemmas}

Each lemma realizes one transform. In this section $\Bs$ is $\Bs(\Phi_{\mathrm{strict}})$ (Definition~\ref{def:phistrict}) and $\Be$ is $\Be(\Phi_{\mathrm{ext}})$ (Definition~\ref{def:phiext}).

\begin{lemma}[Affine maps, residuals, routing, GELU MLP]
\label{lem:basic}
Fixed affine maps, residual sums and routing maps are nodes of both bases. $x\mapsto A_2\gelu(A_1x+b_1)+b_2$ is the composition $\Aff_{A_2,b_2}\circ\gelu\circ\Aff_{A_1,b_1}$ in both bases.
\end{lemma}
\begin{proof}
For the approximation, $x\,\sigma(1.702x)$ is one sigmoid after a fixed scaling and one product: $M\bigl(x,\sigma(1.702x)\bigr)$, two squares per hidden unit, in the strict basis, and one Hadamard node in the extended basis.
\end{proof}

\begin{lemma}[LayerNorm]
\label{lem:ln}
$\LN\in\langle\Bs\rangle$ and $\LN\in\langle\Be\rangle$.
\end{lemma}
\begin{proof}
Let $u=x-\mean(x)\1$, which is affine. Strict: $r^2=\mean_i(u_i^2)$ uses squares and an affine mean, $a=(r^2+\eps)^{-1/2}$ is one unary evaluation per token, and $z_i=\tilde M(u_i,a)=u_i/\sqrt{r^2+\eps}$ reuses $u_i^2$ and forms $a^2$ once per token, so it costs one square per entry and one per token. Then $\LN_i(x)=\beta_i+\gamma_iz_i$ is affine in $z$. Extended: $r^2=\mean(u\odot u)$ is a Hadamard product followed by an affine mean, $a=(r^2+\eps)^{-1/2}$ is one unary evaluation per token, and $\LN(x)=\beta+\gamma\odot(u\odot a)$, where the product with the fixed $\gamma$ is affine and $u\odot a$ is a broadcast Hadamard product.
\end{proof}

\begin{lemma}[Scores and value mixing]
\label{lem:score}
$(Q,K)\mapsto QK^\top/\sqrt d$ and $(P,V)\mapsto PV$ lie in $\langle\Bs\rangle$ and in $\langle\Be\rangle$.
\end{lemma}
\begin{proof}
Strict: $S_{ij}=(2\sqrt d)^{-1}\bigl[\sum_f(q_{if}+k_{jf})^2-\|q_i\|^2-\|k_j\|^2\bigr]$ and $(PV)_{ij}=\tfrac12\bigl[\sum_\ell\bigl(v_{\ell j}^2-(v_{\ell j}-p_{i\ell})^2\bigr)+\sum_\ell p_{i\ell}^2\bigr]$, with the norms formed once per row and the sums affine: $N^2d+2Nd$ squares per head for the scores and $N^2d+N^2+Nd$ for the mix. In the mix the two close squares are subtracted per term, inside the node that sums over $\ell$. Extended: each map is one $\mathrm{MatMul}$ or $\mathrm{BMM}$ node followed by a fixed scaling.
\end{proof}

\begin{lemma}[Softmax]
\label{lem:softmax}
$s\mapsto\softmax(s)$ lies in $\langle\Bs\rangle$ and in $\langle\Be\rangle$.
\end{lemma}
\begin{proof}
Let $t_i=s_i-\max_ls_l\le0$, where the maximum is a routing selection, and $n_i=e^{t_i}$, so that $Z=\sum_jn_j\ge1$. The shift leaves $\softmax$ unchanged and keeps $\sigma(-t_i)$ in $[\tfrac12,1)$ and $Z$ in $[1,N]$, so every inverse square root below acts on a number in $[\tfrac12,N]$. Strict: since $\sigma(t)/\sigma(-t)=e^t$ for every $t$, $n_i=M\bigl(\sigma(t_i),\,(\sigma(-t_i)^{-1/2})^2\bigr)$, with two sigmoids, one inverse square root and three squares per entry; $Z$ is affine in $n$, and $p_i=M\bigl(n_i,(Z^{-1/2})^2\bigr)$ adds two squares per entry and one inverse square root and one square per row. Extended: $n_i=\exp(t_i)$, one exponential per entry, and $p=n\odot Z^{-1}$, one reciprocal per row and a broadcast Hadamard product.
\end{proof}

\begin{lemma}[Tail-mass top-$k$ attention]
\label{lem:topk}
The map of Definition~\ref{def:tail} from $(s,V)$ to $\hat o$, for a fixed calibrated $q$ or for $q=\tau$, lies in $\langle\Bs\rangle$ and in $\langle\Be\rangle$.
\end{lemma}
\begin{proof}
Masks are fixed routing and affine maps. We treat three cases.

\emph{$k=0$.} Definition~\ref{def:tail} gives $\hat o=\frac1N\sum_jv_j$, a fixed affine map.

\emph{$k\ge N$.} Nothing is dropped and the map is dense attention, covered by Lemmas~\ref{lem:score} and~\ref{lem:softmax}.

\emph{$1\le k<N$.} $I=\TopK(s,k)$ and $\hat s=\operatorname{gather}(s,I)$ are routing. In both modes the graph computes the all-values form~\eqref{eq:allvalues}, which equals Definition~\ref{def:tail} because $\sum_{j\in J}v_j=\sum_{j=1}^Nv_j-\sum_{j\in I}v_j$. Its parts are the coefficients $\omega_j=\hat p_j-q/m$, the kept mixing $\sum_{j\in I}\omega_jv_j$, and the tail term $(q/m)\sum_{j=1}^Nv_j$. The kept mixing is the norm expansion of Lemma~\ref{lem:score} over the $k$ kept values in the strict basis and one product node over $k$ values in the extended basis.
In calibrated mode, $\tilde\pi$ follows Lemma~\ref{lem:softmax} on $\hat s$, with the selected maximum $\hat s_1$ as shift; $\hat p_j=(1-q)\tilde\pi_j$ and $\omega_j$ are affine in $\tilde\pi$, and the tail term is a fixed affine map of $V$.
In exact mode, the dense normalizer is formed from all scores with the dense maximum as shift: $n_j=\exp(s_j-\max_ls_l)$, $Z=\sum_jn_j$, $\hat p_j=n_j/Z$ for $j\in I$ (Lemma~\ref{lem:softmax}), and $\tau=1-\sum_{j\in I}\hat p_j$. The coefficients $\omega_j=\hat p_j-\tau/m$ are affine in $(\hat p,\tau)$. The tail term $(\tau/m)\sum_jv_j$ is a product of the scalar activation $\tau$ with the $d$-dimensional activation $\sum_jv_j$: a norm expansion with $\tau^2$ once per row and the squares of $\sum_jv_j$ once per head in the strict basis, and one broadcast Hadamard node with $d$ scalar products in the extended basis.
\end{proof}

\subsection{Proofs of Theorem~\ref{thm:strict}, Propositions~\ref{prop:hsr} and~\ref{prop:hadamard} and Theorem~\ref{thm:ext}}

\begin{proof}[Proof of Theorem~\ref{thm:strict}]
Replace every transform of the block by its strict realization from Lemmas~\ref{lem:basic}--\ref{lem:softmax}. Each realization computes its transform exactly, and $\langle\Bs\rangle$ is closed under composition.
\end{proof}

\begin{proof}[Proof of Proposition~\ref{prop:hsr}]
Take the strict graph of Theorem~\ref{thm:strict}. Replace each product graph, $M(x,y)$ or $\tilde M(x,y)$, by one Hadamard node $x\odot y$, drop the squares of the factors it used, and form each remaining square as $x\odot x$. Since $M(x,y)=\tilde M(x,y)=xy$, the new graph computes the same function, and its unary maps are those of $\Phi_{\odot}$. $\mathrm{Score}$ and $\mathrm{Mix}$ become a broadcast Hadamard product and a fixed sum, $S_{ij}=d^{-1/2}\sum_f(q_{i\cdot}\odot k_{j\cdot})_f$ and $(PV)_{ij}=\sum_\ell(p_{i\cdot}\odot v_{\cdot j})_\ell$, with no unary map. LayerNorm keeps one inverse square root per token, as its source graph. The MLP keeps one GELU per hidden unit, or one sigmoid and one Hadamard product for the approximation, as its source graph. Softmax keeps, per entry, two sigmoids and one inverse square root for $n_i=\sigma(t_i)\odot\bigl(\sigma(-t_i)^{-1/2}\odot\sigma(-t_i)^{-1/2}\bigr)$, and, per row, one inverse square root for $1/Z=Z^{-1/2}\odot Z^{-1/2}$ (Lemma~\ref{lem:softmax}). Its source graph pays one exponential per entry and one reciprocal per row, so softmax adds $2N^2$ unary evaluations per head. For the contraction work, the products of $\mathrm{Score}$ and $\mathrm{Mix}$ cost the $N^2d$ multiplications of the source, and each fixed sum adds at most $N^2d$ multiply--accumulate steps per head. Softmax adds two products per entry for $n_i$ and one per row for $1/Z$, so $2N^2+N$ per head; its sum over the row and its product $n\odot Z^{-1}$ are those of the source graph.
\end{proof}

\begin{proof}[Proof of Proposition~\ref{prop:hadamard}]
Replace every transform except $\mathrm{Score}$ and $\mathrm{Mix}$ by its extended realization from Lemmas~\ref{lem:basic}--\ref{lem:softmax}; these use only Hadamard products among the activation--activation products. Each performs the operations of its source graph (Definition~\ref{def:work}). Affine maps, residuals, routing and the GELU MLP are basis nodes. LayerNorm performs, per token, an affine mean, the product $u\odot u$ and its mean, one inverse square root, and the rescaling. Softmax performs one exponential per entry, an affine sum, one reciprocal per row and a broadcast product. So $\Delta\PW=\Delta\MAC=0$ for these transforms. Compute $\mathrm{Score}$ and $\mathrm{Mix}$ as in the proof of Proposition~\ref{prop:hsr}, by a broadcast Hadamard product and a fixed reduction. No unary map is added, so $\Delta\PW=0$, and the reduction adds at most $N^2d$ multiply--accumulate steps per head.
\end{proof}

\begin{proof}[Proof of Theorem~\ref{thm:ext}]
Replace every transform of the block by its extended realization from Lemmas~\ref{lem:basic}--\ref{lem:softmax}. Each realization computes its transform exactly, and $\langle\Be\rangle$ is closed under composition. Every transform except $\mathrm{Score}$ and $\mathrm{Mix}$ has the realization of the proof of Proposition~\ref{prop:hadamard}, with $\Delta\PW=\Delta\MAC=0$. $\mathrm{Score}$ and $\mathrm{Mix}$ are single $\mathrm{MatMul}$ or $\mathrm{BMM}$ nodes of the source shape: $\PW=0$ and $\MAC=N^2d$ before and after. Hence $\Delta\PW=\Delta\MAC=0$ for every transform.
\end{proof}

The bases $\Bs(\Phi_{\mathrm{strict}})$ and $\Be(\Phi_{\mathrm{ext}})$ are not nested. With $\Phi_{\mathrm{all}}=\Phi_{\mathrm{strict}}\cup\Phi_{\mathrm{ext}}$, Theorems~\ref{thm:strict} and~\ref{thm:ext} also hold in the nested bases $\Bs(\Phi_{\mathrm{all}})\subset\Be(\Phi_{\mathrm{all}})$, since enlarging the unary set keeps every graph valid.

\subsection{Proof of Theorem~\ref{thm:lower}}

The proof follows the nonlinear rank of the target. $\Phi$ is an arbitrary set of real functions. For a set $D\subseteq\R^n$ write $\mathcal A(D)$ for the restrictions to $D$ of affine functions on $\R^n$. For $F\colon D\to\R^r$, the nonlinear rank
\[
\nrk_D(F)=\dim\Bigl(\bigl(\operatorname{span}\{F_1,\dots,F_r\}+\mathcal A(D)\bigr)/\mathcal A(D)\Bigr)
\]
counts the coordinates of $F$ that stay linearly independent on $D$ after adding affine functions.

\begin{lemma}[Fixed routing on a set dense in a ball]
\label{lem:region}
Let $G$ be a finite product-free graph with inputs in $\R^n$, and let $U\subseteq\R^n$ be a nonempty open set. There is a set $D\subseteq U$ and an open ball $B\subseteq U$ such that $D$ is dense in $B$ and every routing decision of $G$ is the same at every point of $D$.
\end{lemma}
\begin{proof}
$G$ has finitely many routing nodes, and each, with its fixed tie-breaking rule, has finitely many outcomes: an order of finitely many compared values, or a selection among finitely many entries. The joint outcome therefore takes finitely many values, and the sets of points of $U$ with a given joint outcome form a finite partition of $U$. A finite union of nowhere dense sets is nowhere dense, and $U$ is not, so some part $D'$ is not nowhere dense: its closure contains an open ball $B$, which we may take inside $U$. Then $D=D'\cap B$ is dense in $B$.
\end{proof}

No continuity is used: the outcome sets may be arbitrary, so $\Phi$ needs no continuity or measurability.

\begin{lemma}[Affine reduction]
\label{lem:affine}
Let $G$ be product-free with outputs $G_1,\dots,G_r$, let all routing of $G$ be constant on a set $D$, and let $y_1,\dots,y_u$ be the outputs of its unary evaluations, $u=\PW(G)$, viewed as functions on $D$. Then there are affine $\ell_j$ and constants $c_{jt}$ with $G_j=\ell_j+\sum_tc_{jt}y_t$ on $D$.
\end{lemma}
\begin{proof}
By induction along a topological order, every value in $G$ is on $D$ an affine function of $(x,y_1,\dots,y_u)$. Inputs are. A fixed affine node maps affine functions to affine functions. A routing node acts on $D$ as a fixed $0/1$ selection matrix. A unary node outputs some $y_t$. No node multiplies two values, so the induction never leaves the affine functions of $(x,y)$.
\end{proof}

\begin{lemma}[Pointwise work bounds nonlinear rank]
\label{lem:nrk}
If a product-free graph $G$ computes $F$ exactly on a set $D$ on which its routing is constant, then $\PW(G)\ge\nrk_D(F)$.
\end{lemma}
\begin{proof}
By Lemma~\ref{lem:affine}, every $F_j$ equals $\sum_tc_{jt}y_t$ modulo $\mathcal A(D)$. In the ambient quotient of all real functions on $D$ by $\mathcal A(D)$, the classes of $F_1,\ldots,F_r$ lie in the span of the $u$ unary-output classes. Their span therefore has dimension at most $u$.
\end{proof}

\begin{lemma}[Rank of quadratic forms]
\label{lem:quad}
Let $D$ be dense in an open ball $B\subseteq\R^n$, and let $F_1,\dots,F_r$ be homogeneous quadratic polynomials on $\R^n$ such that each $F_s$ has a monomial, with nonzero coefficient, that no other $F_t$ has. Then $\nrk_D(F)=r$.
\end{lemma}
\begin{proof}
Suppose $\sum_s\lambda_sF_s$ agrees on $D$ with an affine function. Both sides are continuous and $D$ is dense in $B$, so they agree on $B$. Polynomials that agree on an open set agree everywhere, and a homogeneous quadratic polynomial that is affine is zero. The coefficient of the monomial that only $F_s$ has is $\lambda_s$ times a nonzero number, so every $\lambda_s$ vanishes.
\end{proof}

\begin{proof}[Proof of Theorem~\ref{thm:lower}]
Let $G$ be a product-free graph with inputs in $\R^n$ that computes $S=QK^\top/\sqrt d$ exactly, and take $D$ and $B$ from Lemma~\ref{lem:region} with $U=\R^n$. The entry $S_{ij}=d^{-1/2}\sum_fq_{if}k_{jf}$ is a homogeneous quadratic polynomial, and no other entry has its monomials $q_{if}k_{jf}$. Lemma~\ref{lem:quad} gives $\nrk_D(S)=N^2$, and Lemma~\ref{lem:nrk} gives $\PW(G)\ge N^2$. For $\mathrm{Mix}$, the entry $(PV)_{ij}=\sum_\ell p_{i\ell}v_{\ell j}$, $(i,j)\in[N]\times[d]$, has the monomials $p_{i\ell}v_{\ell j}$, which no other entry has, so Lemmas~\ref{lem:quad} and~\ref{lem:nrk} give $\PW(G)\ge Nd$.
\end{proof}

\begin{proof}[Proof of Corollary~\ref{cor:time}]
A transform-wise product-free compilation of an $h$-head block contains its own $\mathrm{Score}$ and $\mathrm{Mix}$ subgraph for each head, each exact for all arguments and sharing no unary evaluation with the others, so these subgraphs have $\PW\ge h(N^2+Nd)$ by Theorem~\ref{thm:lower}. Unary maps run only on the pointwise unit, at rate at most $R_{\mathrm{pw}}$, so they take at least $h(N^2+Nd)/R_{\mathrm{pw}}$. In the extended graph, $\mathrm{Score}$ and $\mathrm{Mix}$ have $\PW=0$ and $\MAC=2hN^2d$ and, by assumption, take $2hN^2d/R_{\mathrm{mac}}$. The first time exceeds the second exactly when $\kappa>2N^2d/(N^2+Nd)=2Nd/(N+d)$, which is below $2d$. The strict graph of Theorem~\ref{thm:strict} has $\PW=h(2N^2d+N^2+3Nd)>2hN^2d$ in $\mathrm{Score}$ and $\mathrm{Mix}$ (Lemma~\ref{lem:score}), hence a pointwise time of at least $\kappa$ times the extended time.
\end{proof}

\subsection{Proofs for clustered attention}

\begin{proof}[Proof of Lemma~\ref{lem:gap}]
The normalizer is at least $\sum_{j\in I}e^{s_j}\ge k\,e^{s_{(k)}}$, and the tail sum $\sum_{j\in J}e^{s_j}$ is at most $m\,e^{s_{(k+1)}}$, so $\tau\le\frac mk e^{s_{(k+1)}-s_{(k)}}$.
\end{proof}

\begin{proof}[Proof of Lemma~\ref{lem:shared}]
Let $R=\max_l\|k_l\|$. For every key, $|q_i\cdot k_j-c\cdot k_j|\le\|q_i-c\|\,\|k_j\|\le\rho R$. For $j\in I$ and $l\in J$, $(c\cdot k_j-c\cdot k_l)/\sqrt d\ge(q_i\cdot k_j-q_i\cdot k_l)/\sqrt d-2\rho R/\sqrt d\ge\Delta_k-2\rho R/\sqrt d>0$, so every key of $I$ scores above every key of $J$ under $c$.
\end{proof}

\begin{proof}[Proof of Proposition~\ref{prop:sparse}]
Exactness: replace softmax and value mixing by the realization of Lemma~\ref{lem:topk} in each basis, and the other transforms as in Theorems~\ref{thm:strict} and~\ref{thm:ext}. Work in the extended basis: for $k=0$ the map is an affine average and for $k\ge N$ it is dense attention, both equal to their source graphs. For $1\le k<N$ the source graph is the all-values form~\eqref{eq:allvalues} with the $q$ of the mode, and the extended realization of Lemma~\ref{lem:topk} is this graph. $\TopK$ and gather are routing. The softmax follows Lemma~\ref{lem:softmax}, with one exponential per kept score in calibrated mode and one per score in exact mode, and one reciprocal per row. The kept mixing is one product node over $k$ values. The tail term is a fixed affine map in calibrated mode and one broadcast product node with $d$ scalar products in exact mode. Hence $\Delta\PW=\Delta\MAC=0$ for every transform in both modes.
\end{proof}

\begin{proof}[Proof of Theorem~\ref{thm:tail}]
(i) Given $O$, the normalizer $Z=\sum_je^{s_j}$ is a symmetric function of the tail scores, so the tail weights $\pi_j=e^{s_j}/Z$, $j\in J$, are exchangeable given $O$, and $\E[\pi_j\mid O]$ is the same for all $j\in J$. These $m$ values sum to $\E[\tau\mid O]$, so each equals $\E[\tau\mid O]/m$. For $j\in I$, $\pi_j=(1-\tau)\tilde\pi_j$ and $\tilde\pi_j$ is a function of $O$, so $\E[\pi_j\mid O]=(1-\E[\tau\mid O])\tilde\pi_j$. With $q=\E[\tau\mid O]$ this is $\hat p=\E[\pi\mid O]$. Since $V$ is part of $O$ and $o$ is linear in $\pi$, $\E[o\mid O]=\sum_j\E[\pi_j\mid O]v_j=\hat o$. The weights lie in $[0,1]$ and $\E\|v_j\|^2<\infty$, so $\pi$ and $o$ are square integrable, and a conditional expectation of a square-integrable variable is unbiased and minimizes the mean squared error over all functions of the conditioning variables. In exact mode, $\tau$ is observed. Given $O$, $\tau$ is a symmetric function of the tail scores, so they stay exchangeable given $(O,\tau)$. Then $\E[\pi_j\mid O,\tau]$ is the same for all $j\in J$ and sums to $\tau$ over $J$, so it equals $\tau/m$, and $\E[\pi_j\mid O,\tau]=(1-\tau)\tilde\pi_j$ for $j\in I$. This is $\hat p$ with $q=\tau$.
(ii) For $q=0$, $\hat p_I=\tilde\pi_I$ and $\hat p_J=0$, so $\E[\hat p-\pi\mid O]=(\tilde\pi_I-(1-\E[\tau\mid O])\tilde\pi_I,\,-\E[\tau\mid O]\1_J/m)=\E[\tau\mid O](\tilde\pi_I,-\1_J/m)$. By (i), $\E[o\mid O]=(1-\E[\tau\mid O])a+\E[\tau\mid O]\bar u$, and $\hat o=a$ for $q=0$, so $\E[\hat o-o\mid O]=\E[\tau\mid O](a-\bar u)$.
\end{proof}

\begin{proof}[Proof of Proposition~\ref{prop:error}]
By Definition~\ref{def:tail}, $\hat o=(1-q)a+q\bar u$, and $o=(1-\tau)a+\tau b$. Subtracting, $\hat o-o=(\tau-q)a+q\bar u-\tau b=\tau(\bar u-b)+(q-\tau)(\bar u-a)$. For $q=0$ this is $\tau(a-b)$. For the distribution, $\hat p_j-\pi_j=(\tau-q)\tilde\pi_j$ on $I$, and $\hat p_j-\pi_j=(q-\tau)/m+(\tau/m-\pi_j)$ on $J$. Squaring and summing, the cross term is $2(q-\tau)m^{-1}\sum_{j\in J}(\tau/m-\pi_j)=0$, since $\sum_{j\in J}\pi_j=\tau$, which leaves $(q-\tau)^2(\|\tilde\pi_I\|_2^2+1/m)+\sum_{j\in J}(\pi_j-\tau/m)^2$.
\end{proof}

\begin{proof}[Proof of Corollary~\ref{cor:sparse}]
Choose an open set $U$ of inputs on which the selection is constant, take $D$ and $B$ from Lemma~\ref{lem:region} inside $U$, and write the selected key indices of query $i$ as $j_{i1},\ldots,j_{ik}$. The selected score $F_{it}=d^{-1/2}\sum_f q_{if}k_{j_{it}f}$ has the monomials $q_{if}k_{j_{it}f}$, which no other selected score has, since the $k$ keys of a query are distinct. Lemmas~\ref{lem:quad} and~\ref{lem:nrk} give $\PW\ge Nk$. For the upper bound, routing gathers the $k$ keys of each query, each selected score costs $d$ squares of $q_i+k_j$, and the norms of the $N$ queries and the $N$ keys cost $2Nd$ (Lemma~\ref{lem:score}).
\end{proof}

\begin{proof}[Proof of Proposition~\ref{prop:routed}]
The assignment scores $q_i\cdot c_r-\tfrac12\|c_r\|^2$ and the centroid scores $c_r\cdot k_j$ are fixed affine maps of $Q$ and $K$, with $hCNd$ multiply--accumulate steps each. $\TopK$ over centroids, $\TopK$ over keys for each centroid, and the gather of each query's list are routing. The selected scores cost $Nkd+2Nd$ squares per head (Corollary~\ref{cor:sparse}), and the kept mixing of Lemma~\ref{lem:topk} costs $Nkd+Nk+Nd$ (Lemma~\ref{lem:score}). With a constant tail mass the tail term is a fixed affine map of $V$. In the extended basis the scores and the mixing are product nodes over $k$ keys, and every unary map is one of the source graph.
\end{proof}

\begin{proof}[Proof of Proposition~\ref{prop:flock}]
$g(\|x_i-x_j\|)=e^{-\theta\|x_i\|^2}\,e^{2\theta x_i^\top x_j-\theta\|x_j\|^2}$. The first factor does not depend on $j$ and cancels against the same factor in the sum over $l$, which leaves the softmax. The score of agent $j$ equals $\theta\|x_i\|^2-\theta\|x_i-x_j\|^2$, whose first term does not depend on $j$; since $\theta>0$, the $k$ largest scores belong to the $k$ nearest agents. Lemma~\ref{lem:shared} holds for these scores with $\sqrt d$ replaced by $1$: the bias of agent $j$ enters the scores of every query equally and cancels in each comparison of its proof.
\end{proof}

\subsection{Models and data}

\emph{DeiT-Tiny.} We use the timm~\cite{wightman2019timm} DeiT-Tiny~\cite{touvron2021deit} (patch 16, $224\times224$ input, $N=197$). For Oxford-IIIT Pet~\cite{parkhi2012pet} (37 classes, 3{,}680 training and 3{,}669 test images) we fine-tune the classifier head for 20 epochs on the training split and keep the epoch with the highest training accuracy, measured without augmentation; the reference reaches $91.36\%$. Pet has no validation split. We report its test split, and no step selects on it. For ImageNet-1k~\cite{deng2009imagenet} we use the pretrained checkpoint ($72.02\%$).

\emph{MambaIRv2-Light.} We use the released $\times2$ and $\times4$ lightweight super-resolution weights~\cite{guo2025mambairv2}. The attention acts on $16\times16$ windows, so $N=256$. Training uses DIV2K~\cite{agustsson2017div2k} patches of $64\times64$ low-resolution pixels; evaluation reports PSNR on the Y channel of Set5~\cite{bevilacqua2012set5} with the standard border crop.

\emph{Pythia-70M.} We use EleutherAI's Pythia-70M~\cite{biderman2023pythia} with eager attention on non-overlapping 128-token windows of WikiText-2 (raw)~\cite{merity2017wikitext}, $2.53\times10^{5}$ validation tokens. The partial rotary embedding multiplies each query and key by a fixed rotation for each position, an affine node.

\subsection{Tail-mass calibration}

Calibrated mode runs the reference model on calibration minibatches from the training split and records, for each block, the dense-softmax mass outside the top-$k$ scores of every query. The block's $q$ is the mean over queries. We use the full Pet training split (115 minibatches of 32), 10 ImageNet training minibatches and 8 DIV2K minibatches. Exact mode computes $\tau$ for each query at run time from the full softmax normalizer. Supplementary Section~S6 compares the tail modes and the calibrated with the exact tail mass.

\subsection{Query-centroid routing}

For each block and head, k-means (k-means++ initialization, 30 Lloyd iterations) fits 32 centroids $c_r$ to at most $2^{16}$ queries from 20 random training minibatches (64 for Pythia-70M), taken after the rotary embedding. Each query goes to its nearest centroid and takes the top $k$ keys of $c_{r(i)}\cdot k_j$, under the causal mask for Pythia-70M. Without a causal mask, each centroid keeps its top $k$ keys and each query takes the list of its centroid, so no $N\times N$ matrix is formed. For causal routing, the implementation materializes an $N\times N$ matrix of centroid selection scores and forms valid-value means with a dense lower-triangular contraction; keeping each centroid's top-$k$ keys and the value sum over successive prefixes removes both dense operations. In calibrated mode and with tail drop the graph gathers the $k$ routed keys of each query and forms only their scores, by the norm expansion in the strict graph and by a batched product in the extended graph. Exact mode forms all scores for its normalizer. The tail term is that of Definition~\ref{def:tail}. For the strict graph with GELU at $k=16$, $32$ and $64$, the latency benchmark gives $4.19$, $3.89$ and $5.41$ times the reference latency with the routed support against $4.34$, $4.58$ and $6.19$ with exact top-$k$; the extended graph gives $3.35$ to $4.37$ against $3.28$ to $4.88$, and its peak memory changes by less than $5\%$. Calibrated mode uses the mean mass outside the routed support on the calibration minibatches of the top-$k$ runs. We report, on the held-out Pet test split and the WikiText-2 validation split, the dense mass inside the routed support relative to the exact top-$k$ support, and the accuracy before fine-tuning. Pythia-70M has no calibrated mode, so its perplexity is compared with tail dropping in both supports. With 8 and 128 centroids the routed mass at $k=16$ is $76\%$ and $92\%$ of the top-$k$ mass on DeiT-Tiny and $89\%$ and $93\%$ on Pythia-70M.

\subsection{Distillation and fine-tuning}

For DeiT-Tiny on Pet we distill~\cite{hinton2015distilling} the compiled model against the frozen reference for one epoch with AdamW (batch 64, learning rate $5\times10^{-5}$, weight decay $0.05$, cosine schedule without warm-up, gradient clipping at $1.0$) and loss $(1-\lambda)\,\mathrm{CE}(s,y)+\lambda\bigl(\mathrm{KL}(s\|t)+\mathrm{KL}(t\|s)\bigr)$, with $\lambda=1$ (Jeffreys) or $\lambda=0$ (cross-entropy), over seeds 42 to 46. Supplementary Section~S7 compares the two losses. Each run keeps the epoch, or the model before training, with the highest training accuracy, measured without augmentation, and later steps use the seed with the highest such accuracy; ties keep the earlier one. The tail mass $q$ is trainable, a scalar parameter per block clamped to $[0,1]$ in the forward pass. For MambaIRv2-Light we fine-tune for three epochs with AdamW (batch 8, learning rate $2\times10^{-4}$, weight decay $10^{-4}$, cosine schedule, gradient clipping at $1.0$) on the $L_1$ loss to the high-resolution targets and keep the epoch with the highest training PSNR on one center crop of each DIV2K training image: the third epoch in 22 of the 26 runs, the second in 3, and the untrained model for $\times2$ at $k=256$. Set5 is only evaluated. Pythia-70M is evaluated without training.

\subsection{Post-training quantization}

\emph{Precision.} We target edge NPUs, so DeiT-Tiny and MambaIRv2-Light run in bf16, which stores each value in half the memory of fp32 with 8 significant bits. With all keys, bf16 changes the accuracy of every DeiT-Tiny graph by at most $0.46$ points and the PSNR of MambaIRv2-Light by at most $0.04$\,dB against fp32 (Supplementary Section~S8). An 8-bit PTQ code drops the 8-bit exponent of a bf16 value and keeps its sign and 7 more bits: all values of a tensor share one scale. Near the largest value of its tensor the code is as fine as bf16, and smaller values lose bits. A square doubles the exponent range of its input: a value at $2^{-4}$ of the largest has a square at $2^{-8}$ of the largest square. Bit widths are those of the input codes. The fixed affine nodes of the strict graph read squares and sums of squares, so their inputs take 12 bits; weights, the logits and the inputs of linear, convolution and product nodes, including the two scalings by the calibrated tail mass in each block, take 8 bits.

\emph{Quantized graph.} Quantized nodes are the linear and convolution layers, the products of two activations or of an activation and a parameter, and the fixed affine nodes with a coefficient tensor, including the residual and position-embedding additions. Each quantized node has one quantizer per input, and this quantizer is also the output quantizer of the preceding node, so every boundary is quantized once. An identity node at the model output gives the last layer an output quantizer too, so the logits are 8-bit codes with one scale, and the model is quantized from the image to the logits. Unary maps, routing, and plain sums and means over an axis have no quantizer of their own. A sum or mean reads integer codes with fixed scales, or their squares, so its result is an integer with a fixed scale. It keeps its full integer range, with no rounding to a specific bit width; the next quantizer sets the width. The graph between two quantizers is thus a unit with integer inputs and integer outputs: a quantized node with its output map, followed by any unary, routing and summing nodes. Every tensor that passes from one unit to the next is an integer code, so the quantized graph is an integer graph however its units are evaluated. On an NPU, a unit is an integer multiply--accumulate followed by a requantization, a fixed multiplier and offset~\cite{jacob2018quantization}, and the lookup-table or piecewise-linear units of the NPU approximate its unary nodes over the calibrated input range~\cite{yu2022nnlut,upadhyay2024nova}.

\emph{Simulation and calibration.} We simulate each unit in fp32~\cite{nagel2021white}, with the integer codes held in fp32 tensors. fp32 holds every integer up to $2^{24}\approx1.7\times10^{7}$, so the integer sums are exact. The run checks this for every node with integer inputs: over the calibration data, the sum of the absolute products of the node stays below $4.5\times10^{6}$, a factor $3.7$ below this limit. The output maps and unary nodes run in fp32. The bf16 float model is the reference. Calibration runs it for two forward passes over the full Pet training split: the first records the maximum absolute value of every node input, the second fits the output maps to its outputs. Activations use one symmetric max-abs scale per input tensor. Linear and convolution weights use one symmetric scale per output channel; affine coefficient tensors use one scale. A linear or convolution node takes its output scale as the product of input and weight scales and fits a per-channel output bias; other nodes fit an affine output map by least squares. A node that combines several inputs of different range, as the norm expansion does with $v^2$ and $(v-p)^2$ or with $u^2$ and $a^2$, quantizes each input on its own grid and rescales it before the sum, as an integer add with per-input multipliers does. Quantizing the logits changes accuracy by $-0.16$ to $+0.30$ points. Supplementary Section~S9 gives PTQ of the strict graph with exact GELU.

\emph{Integer products from squares.} For integer operands both identities hold exactly when no intermediate step overflows or requantizes. For signed 8-bit operands in $[-128,127]$, the sums and differences lie in $[-256,255]$. A shared table of squares needs 512 signed-index entries, or 257 magnitude-index entries, with 17-bit unsigned outputs, and the difference of two squares fits in 18 signed bits. Dividing $(a+b)^2-(a-b)^2$ by four, or $(a+b)^2-a^2-b^2$ by two, before accumulation gives the integer product exactly. Each product has magnitude at most $2^{14}$, so a signed 32-bit accumulator holds a dot product of every length $d<2^{17}$.

\subsection{Timing and unary counts}

Latency is the GPU time of the forward pass alone, divided by that of the reference, on one NVIDIA RTX~4070~Ti~SUPER. Ten batches of 32 images from the held-out Pet test split are first stored on the GPU. Each model runs in bf16 under \texttt{torch.compile} in the default, reduce-overhead and max-autotune modes, and we report its fastest mode: the median of five timed runs over the ten batches, with CUDA events after a warm-up. Peak memory is the maximum allocated GPU memory during an eager evaluation pass over the held-out Pet test split, with only the measured model on the GPU; Supplementary Section~S10 gives the tensor sizes behind it. Unary evaluations are counted with forward hooks on every unary node over one test minibatch. With all keys, the DeiT-Tiny count per image equals twelve block totals of Table~\ref{tab:pw} plus the final LayerNorm: $6.02\times10^{6}$ for the Hadamard graph and $1.96\times10^{8}$ for the strict graph with GELU. The implementation of the extended graph forms the squares $u\odot u$ of its 25 LayerNorms with a unary square node, so its count of $4.17\times10^{6}$ adds $25Nhd=9.46\times10^{5}$ to that sum. With $x\odot\sigma(1.702x)$ the strict graph forms this product from two squares, $8Nhd$ more per block.

\emph{The GPU ratio.} The RTX~4070~Ti~SUPER has 66 Ada streaming multiprocessors at a boost clock of $2.61$\,GHz. Each multiprocessor retires 128 FP32 operations per clock on its CUDA cores and 256 dense bf16 multiply--accumulates per clock with FP32 accumulation on its tensor cores~\cite{nvidia_ada}. This gives $R_{\mathrm{pw}}\approx22$\,T/s, $R_{\mathrm{mac}}\approx44$\,T/s and $\kappa\approx2$.

\subsection{Numerical conventions}

The compiled LayerNorm uses the $\eps$ of each reference: $10^{-6}$ for DeiT-Tiny and $10^{-5}$ for MambaIRv2-Light and Pythia-70M. Softmax normalization divides by the max-shifted partition sum $Z$. The DeiT-Tiny and MambaIRv2-Light checkpoints store their weights in bf16; Pythia-70M runs in fp32 (Supplementary Section~S8). The bf16 runs evaluate the reference and every compiled graph under PyTorch CUDA autocast, which keeps means and sums in fp32; a node that adds several inputs stores its result in the precision of its first input. In float64, against the reference with its weights rounded to bf16 the same way, the extended and strict DeiT-Tiny graphs and the Hadamard graph with exact GELU, all with all keys, reproduce the reference logits on 16 Pet test images to $6\times10^{-8}$. The differences between the compiled and reference models in Table~\ref{tab:strict-fast} therefore come from bf16 weights and arithmetic. The implementation of the strict graph evaluates $(x+\eps)^{-1/2}$ as a single kernel; it equals the inverse square root after a fixed affine map.

\section{Data availability}
All datasets are public: Oxford-IIIT Pet~\cite{parkhi2012pet}, ImageNet-1k~\cite{deng2009imagenet}, DIV2K~\cite{agustsson2017div2k}, Set5~\cite{bevilacqua2012set5} and WikiText-2~\cite{merity2017wikitext}. Pretrained weights come from timm, the MambaIRv2 release and the Hugging Face hub.

\section{Code availability}
The compiler, the experiment configurations and the scripts that produce every table and figure are released under the MIT license at \url{https://github.com/iceslice0/transformer-to-dnn}.

\bibliographystyle{naturemag}
\bibliography{main}

\clearpage
\begin{center}
{\Large\sffamily\bfseries Supplementary Information}
\end{center}

\setcounter{table}{0}
\setcounter{figure}{0}
\setcounter{theorem}{0}
\setcounter{lemma}{0}
\renewcommand{\thetable}{S\arabic{table}}
\renewcommand{\thefigure}{S\arabic{figure}}

\section{S1. Tightness and scope of the bound}

\emph{Squares only.} With $\Phi=\{(\cdot)^2\}$ and one nonlinear layer, $\PW$ is within a factor two of the bilinear rank of the matrix multiplication tensor, since polarization turns each product into two squares and each square is one product. Fast matrix multiplication~\cite{strassen1969,alman2025asymmetry} then gives strict graphs with $\PW=O(n^{2.373})$ for $n\times n$ products, so no $\Omega(N^2d)$ bound holds for this $\Phi$, and the factor $d+2d/N$ between the strict graph of Theorem~\ref{thm:strict} and Theorem~\ref{thm:lower} is not forced. Conversely, on the set of Lemma~\ref{lem:region} a graph over $\{(\cdot)^2\}$ with $\PW$ squarings is an arithmetic circuit with $\PW$ nonscalar multiplications, and nonscalar multiplications never exceed the total arithmetic cost~\cite{burgisser1997act}. For $d=N$, a lower bound of order $N^{2+\delta}$ for this $\Phi$ would therefore prove that the exponent of matrix multiplication is at least $2+\delta$, an open problem. For fixed $d$, $N^2d=O(N^2)$ and the theorem is tight in order. MAC arrays and tensor units run the classical product, and the strict graph follows it.

\emph{Unrestricted $\Phi$.} On a bounded box, $\mathrm{Score}$ is computed by a product-free graph with $N^2+2Nd$ unary evaluations, so no argument that ignores the regularity of $\Phi$ can improve Theorem~\ref{thm:lower} beyond the additive $2Nd$. Map the $n=2Nd$ inputs affinely into $[0,1)$. For $p=1,\dots,n$ let $\psi_p$ send $x\in[0,1)$, written in binary with infinitely many zero digits, to the number whose digit at position $rn+p$ is the $r$-th digit of $x$ and whose other digits are zero. The affine sum $z=\sum_p\psi_p(x_p)$ interleaves the digits of all inputs, has infinitely many zero digits, and so determines every $x_p$. One further unary map per output, $\phi_{ij}(z)=S_{ij}$, decodes the score. The total is $n+N^2$ evaluations. The construction reaches the bound by first encoding its inputs; one evaluation of a linear form alone cannot compute a product on an open set, because $\phi(a\cdot x+b)$ is constant on hyperplanes.

\emph{Transform-wise scope.} Applied to a whole block as one map, the nonlinear rank is at most the number of outputs, $Nhd=3.78\times10^{4}$ for DeiT-Tiny, below the pointwise work of the source block. A compilation that never forms $QK^\top$ is therefore outside Theorem~\ref{thm:lower}. The theorem concerns the attention contractions as transforms, which is the form in which every Transformer implementation and every NPU operator set presents them.

\section{S2. NPU operator support}

Table~\ref{tab:s-npu} answers, for ten edge accelerators and from their documented operator support, whether the dense strict, Hadamard and extended graphs run entirely on the NPU. ``Whole'' means the documentation lists every operation of the graph as supported on the accelerator; otherwise the cell lists the operations left to the host CPU or GPU, or, on R-Car, to the on-chip DSP. MatMul is a matrix product of two activations, as in $QK^\top$ and $PV$, and Rsqrt is the inverse square root. Where the compiler or runtime is open source we read the placement from its code at the commit given in the reference; for the three closed toolchains, Rockchip, Jetson DLA and Edge TPU, we use the vendor documentation. A square runs as the elementwise product $x\odot x$, which every accelerator in the table supports. A softmax counts as on the NPU when the NPU runs it as a native row-wise operator or runs both the exponential and the reciprocal. Where an accelerator lacks GELU, the MLP uses $x\,\sigma(1.702x)$, which every accelerator in the table runs with its sigmoid and its elementwise product; edge deployments also replace GELU by ReLU and fine-tune~\cite{mirzadeh2024relu}. Compiling the graphs with these toolchains and running them on the devices is future work.

\begin{table}[!t]
\caption{\textbf{Documented NPU support of the dense compiled graphs.} Whole-graph compilation on the devices is future work.}
\label{tab:s-npu}
\centering
\small
\setlength{\tabcolsep}{4pt}
\renewcommand{\arraystretch}{1.15}
\begin{tabular}{@{}>{\raggedright\arraybackslash}p{3.3cm} >{\raggedright\arraybackslash}p{3.7cm} >{\raggedright\arraybackslash}p{2.0cm} >{\raggedright\arraybackslash}p{2.0cm} >{\raggedright\arraybackslash}p{3.8cm}@{}}
\toprule
Accelerator & Unary maps on the NPU & Strict graph & Hadamard graph & Extended graph\\
\midrule
Arm Ethos-U55/U65~\cite{arm_vela_code} & Exp, Log, Rsqrt, GELU, Sigmoid, Tanh & Whole & Whole & Whole\\
Arm Ethos-U85~\cite{arm_vela_code} & Exp, Log, Rsqrt, GELU, Sigmoid, Tanh & Whole & Whole & Whole\\
VeriSilicon VIP~\cite{verisilicon_vxdelegate} & Exp, Log, Rsqrt, GELU, Sigmoid, Tanh & Whole & Whole & Whole\\
TI TDA4VM, C7x with MMA~\cite{ti_tidl_ops} & Exp, Log, Sqrt, GELU, Sigmoid, Tanh; power by a constant & Whole & Whole & Whole\\
Rockchip RK3588~\cite{rknn_ops} & GELU, Sigmoid, Tanh; power with exponent $\pm0.5$, 2 or 3 & Whole & Whole & Whole\\
Renesas R-Car V4H/V4M, CNN-IP and DSP~\cite{renesas_hyco} & CNN-IP: ReLU, Sigmoid, Tanh & DSP: Rsqrt & DSP: Rsqrt & DSP: MatMul, softmax, LayerNorm, GELU\\
\midrule
NVIDIA Jetson Xavier/Orin DLA~\cite{nvidia_dla} & ReLU, Sigmoid, Tanh, Abs, Sin, Cos, Atan; INT8 power and division by a constant & Whole & Whole & Host: MatMul\\
Google Edge TPU~\cite{coral_edgetpu_models} & ReLU, Sigmoid, Tanh, Rsqrt & Whole & Whole & Host: MatMul, row-wise softmax\\
Open-source NVDLA IP~\cite{nvdla_sw} & ReLU, Sigmoid, Tanh & Host: Rsqrt & Host: Rsqrt & Host: MatMul, softmax, LayerNorm Rsqrt\\
Arm Ethos-N78~\cite{arm_ethosn} & ReLU, Sigmoid, Tanh & Host: Rsqrt & Host: Rsqrt & Host: MatMul, softmax, LayerNorm Rsqrt\\
\bottomrule
\end{tabular}
\end{table}

The strict graph needs only squares, the sigmoid and the inverse square root, so it runs whole on seven accelerators: Ethos-U55/U65, Ethos-U85, VeriSilicon VIP, the TI TDA4VM, the Rockchip RK3588, the Edge TPU and, in INT8, the Jetson DLA. On the Jetson DLA the inverse square root runs as an INT8 power and the reciprocal as an INT8 division, each with a constant operand, and softmax runs on the NPU on Orin. The R-Car CNN-IP lacks the inverse square root and runs it on the on-chip DSP. The open-source NVDLA IP and Ethos-N78 lack the inverse square root too and leave it, with the LayerNorm and the reciprocal that need it, to the host. The extended graph runs whole on Ethos-U, VeriSilicon, the TDA4VM and the RK3588, and on R-Car with the DSP. The toolchains of the other four have no matrix product of two activations and place the MatMul on the host, on three of them also softmax, and on two also the inverse square root of LayerNorm. The Edge TPU compiler adds a limit of its own: it splits a model only once, so the first unsupported operation and everything after it run on the CPU~\cite{coral_edgetpu_models}. The strict graph runs whole there, and the extended graph moves to the CPU from the first attention score on. On NVIDIA these limits come from the toolchains, not the hardware: the open-source NVDLA IP, which the Jetson DLA is based on, has a programmable lookup table with linear interpolation, two tables of 65 and 257 entries, in its activation and normalization units~\cite{nvdla_lut}. It can hold every function of $\Phi_{\mathrm{strict}}$ and $\Phi_{\mathrm{ext}}$, but both NVIDIA toolchains fill it only for the functions in the table.

Every accelerator multiplies two activations elementwise, so none is product-free in the sense of Definition~\ref{def:strict}: a CNN-class NPU provides the Hadamard basis $\Bh$. There, $QK^\top$ and $PV$ need no pointwise work, since a broadcast Hadamard product followed by a fixed sum computes them (Proposition~\ref{prop:hadamard}). With the sigmoid and the inverse square root, the Hadamard basis also runs the rest of the block (Proposition~\ref{prop:hsr}). The Hadamard graph takes the operations of the strict graph, with the elementwise product between two activations where the strict graph multiplies an activation by itself, so it runs on the accelerators where the strict graph runs. Where the toolchain multiplies only tensors of equal shape, as on the Ethos-N78, broadcast routing expands both operands first, as the strict graph does for its polarization operands. The Ethos-U55 compiler takes this route: it turns a matrix multiplication between activations into an elementwise product into an int32 buffer followed by a sum, one output column at a time. The Ethos-U85 compiler runs the same operation on its MAC array. $\TopK$ runs on the TDA4VM accelerator and on the R-Car DSP. On the other accelerators top-$k$ attention and the routed support keep the selection on the host; Ethos-U and VeriSilicon can instead replace $\TopK$ by $k$ arg-max passes with masking.

\section{S3. Why the bound counts pointwise work}

Write $\mathrm{IO}_L(G)$ for the minimum number of words moved between a fast memory of $L$ words and slow memory, over all schedules of a graph $G$. With bandwidth $\beta$, every schedule takes time
\[
T(G)\ \ge\ \max\Bigl(\mathrm{IO}_L(G)/\beta,\ \MAC(G)/R_{\mathrm{mac}},\ \PW(G)/R_{\mathrm{pw}}\Bigr).
\]
The unfused strict graph writes the $N^2d$ values $q_{if}+k_{jf}$ to memory (Supplementary Section~S10), so this one schedule has arithmetic intensity $\Theta(1)$ and moves far more than $\mathrm{IO}_L$. These writes are a property of the schedule, not of the basis. The strict expansion of $\mathrm{Score}$ in Lemma~\ref{lem:score} has the access pattern of the matrix product it replaces, so the tiling of the matrix product applies. With square output tiles of side $\Theta(\sqrt L)$, streaming over the reduction index, and $L$ small compared with the operands, a fused schedule moves
\[
\mathrm{IO}_L=O\Bigl(\frac{N^2d}{\sqrt L}+N^2+Nd\Bigr)
\]
words per head, where the last two terms are the compulsory traffic of writing the scores and reading $Q$ and $K$; for classical (non-Strassen) schedules of the native product the first term is also a lower bound in this regime~\cite{hong1981io}. A fused strict kernel therefore has the same traffic order as native $QK^\top$. Fusion changes neither the number of unary evaluations nor the unit that must perform them. That is why Theorem~\ref{thm:lower} bounds $\PW$ and Corollary~\ref{cor:time} rests on the third term. The same tiling applies to a Hadamard product followed by a fixed reduction (Proposition~\ref{prop:hadamard}): fused, it has the access pattern of the matrix product, so the Hadamard and extended bases differ only in the schedule a compiler chooses.

\section{S4. Where the compiled graph spends its time}

Table~\ref{tab:s-profile} runs each graph eagerly, times the top-level submodules of every block with CUDA events and records their working memory, for DeiT-Tiny on Pet at $k=16$ (bf16, 20 test batches of 32, RTX~4070~Ti~SUPER). The working memory of a module is the peak memory it allocates above what is in use when it starts; the table reports the largest value over its calls. The strict graph adds $64.8$\,ms per batch over the extended graph; $73.3\%$ of it falls in $\mathrm{Score}$ and $\mathrm{Mix}$, and $9.8$\,ms in softmax. The module totals are $117.0$ and $52.2$\,ms per batch, a ratio of $2.2$. The profiled reference runs its attention unfused, with $QK^\top$, softmax and $PV$ as separate kernels: it spends $9.0$\,ms there, against $4.1$\,ms and $2$\,MiB in its fused attention kernel, with a module total of $25.9$\,ms. At $k=32$ and $k=64$ the shares are $74.8\%$ and $76.5\%$. $\mathrm{Score}$ costs the same at every $k$ because scores are computed densely; $\mathrm{Mix}$ grows with $k$. The routed graph of Table~\ref{tab:strict-fast} forms only the selected scores: $\mathrm{Score}$ falls from $50$ to $11$\,ms per batch and its working memory from $469$ to $77$\,MiB, and the centroid selection adds $9.6$\,ms.

The Hadamard graph of Proposition~\ref{prop:hsr} forms the same broadcast grid as the strict graph, with one product per term instead of a sum and a square. Its $\mathrm{Score}$ takes $31.6$\,ms per batch, against $49.7$\,ms for the strict graph, and $471$\,MiB, the bf16 tensor of the $N^2d$ products per head of one batch; the strict graph holds one tensor of the same size, since it squares its sums in place. Its module total is $78.2$\,ms. With routed support its $\mathrm{Score}$ falls to $9.2$\,ms and $78$\,MiB, the centroid selection adds $9.7$\,ms, and the module total is $65.2$\,ms.

The profile runs one kernel per node and synchronizes around every module, so it shows where eager time goes. Table~\ref{tab:strict-fast} times the compiled forward pass on batches already on the GPU (Methods), where \texttt{torch.compile} fuses the elementwise chains of every graph. At $k=16$ the compiled reference takes $2.20$\,ms per batch, and the extended, Hadamard and strict graphs take $7.21$, $9.18$ and $11.08$\,ms. The strict graph thus takes $1.54$ times the time of the extended graph, against $2.2$ in the eager module totals. Fused kernels keep the intermediates of a chain in fp32; all accuracies come from the eager bf16 runs.

\begin{table}[h]
\caption{\textbf{Per-module GPU time and working memory at $k=16$} in eager execution, summed over the 12 blocks of one batch. Routed: 32 query centroids. Ext.: the extended graph, with GELU; the other graphs use $x\odot\sigma(1.702x)$. Ref.: the reference with unfused attention. Softmax: top-$k$ with the tail term, dense for Ref.}
\label{tab:s-profile}
\centering
\small
\setlength{\tabcolsep}{3pt}
\begin{tabular}{l rrrrrr rrrrrr}
\toprule
& \multicolumn{6}{c}{Time (ms/batch)} & \multicolumn{6}{c}{Working memory (MiB)}\\
\cmidrule(lr){2-7}\cmidrule(lr){8-13}
& \multicolumn{2}{c}{Strict} & \multicolumn{2}{c}{Hadamard} & & & \multicolumn{2}{c}{Strict} & \multicolumn{2}{c}{Hadamard} & &\\
\cmidrule(lr){2-3}\cmidrule(lr){4-5}\cmidrule(lr){8-9}\cmidrule(lr){10-11}
Module & top-$k$ & routed & top-$k$ & routed & Ext. & Ref. & top-$k$ & routed & top-$k$ & routed & Ext. & Ref.\\
\midrule
$\mathrm{Score}$ ($QK^\top$) & 49.7 & 10.7 & 31.6 & 9.2 & 4.0 & 4.1 & 469 & 77 & 471 & 78 & 14 & 14\\
$\mathrm{Mix}$ ($PV$) & 7.4 & 11.8 & 4.9 & 7.0 & 5.6 & 2.8 & 81 & 81 & 78 & 78 & 39 & 12\\
LayerNorm & 23.0 & 22.2 & 17.9 & 17.9 & 17.9 & 5.5 & 16 & 17 & 19 & 19 & 19 & 9\\
Softmax & 24.9 & 24.0 & 14.0 & 12.7 & 15.1 & 2.1 & 15 & 18 & 9 & 11 & 8 & 28\\
Centroid selection & -- & 9.6 & -- & 9.7 & -- & -- & -- & 4 & -- & 3 & -- & --\\
Linear (qkv, proj) & 3.3 & 3.8 & 4.2 & 3.6 & 4.3 & 5.0 & 7 & 7 & 9 & 9 & 9 & 9\\
MLP & 8.7 & 9.4 & 5.5 & 5.1 & 5.3 & 6.3 & 83 & 83 & 28 & 28 & 18 & 18\\
\bottomrule
\end{tabular}
\end{table}

\section{S5. Polarization and squares}

The strict graph forms each product of two activations from squares, by one of two classical identities:
\[
ab=\tfrac14\bigl[(a+b)^2-(a-b)^2\bigr],\qquad ab=\tfrac12\bigl[(a+b)^2-a^2-b^2\bigr].
\]
The first, polarization, is the method of quarter-squares~\cite{laundy1856quarter}. The second, the norm expansion, forms products in expressivity proofs for ReLU networks~\cite{yarotsky2017error}.

\emph{Cost.} Polarization pays two squares per product. The norm expansion pays one square per product plus the squares $a^2$ and $b^2$ of the factors, so an isolated product costs two squares by polarization and three by the norm expansion. In a contraction each factor enters many products, and its square is formed once. $QK^\top$ of two $N\times d$ matrices then costs $2N^2d$ squares by polarization and $N^2d+2Nd$ by the norm expansion, and the ratio to the bound of Theorem~\ref{thm:lower} falls from $2d$ to $d+2d/N$. This is Winograd's precomputation of the self-terms of each row~\cite{winograd1968inner}, with squares in place of products. The strict graph uses the norm expansion where the squares of the factors are shared, in $QK^\top$, $PV$, LayerNorm and the exact tail term, and polarization elsewhere, in softmax and the GELU approximation.

\emph{Exactness.} Both identities are exact in real arithmetic and, without overflow or requantization, in integer arithmetic (Methods).

\emph{Rounding.} Both identities recover a small product as the difference of large squares. In the value mix the result $y=\sum_kp_kv_k$ is a weighted mean of the values, while the squares sum to about $\sum_kv_k^2$; on DeiT-Tiny this ratio is $65$ at $k=16$ and $543$ with all keys. In bf16 the order of the operations then sets the error (Table~\ref{tab:s-mix}). Rounding each partial sum before the subtraction, $\tfrac12\bigl[\sum_k(p_k+v_k)^2-\sum_kp_k^2-\sum_kv_k^2\bigr]$, gives $57\%$ relative error with all keys. Subtracting the two close squares per term inside the node that sums over keys gives $14\%$ for either identity: $\tfrac14\sum_k\bigl[(p_k+v_k)^2-(p_k-v_k)^2\bigr]$ by polarization and $\tfrac12\bigl[\sum_k\bigl(v_k^2-(v_k-p_k)^2\bigr)+\sum_kp_k^2\bigr]$ by the norm expansion. With this order the two identities differ in cost, not in precision. As Brent showed for Winograd's identity~\cite{brent1970error}, scaling the operands helps further: the fixed affine maps $p\to\lambda p$ and $v\to v/\lambda$ with $\lambda^2=\|v\|/\|p\|$ cut the error to $2.7\%$.

\emph{Scaling and PTQ.} Brent's scaling and the PTQ scales are the same choice. PTQ gives each tensor a grid with step $s=\max|x|/(2^{b-1}-1)$ for $b$ bits. The node that forms $a\pm b$ adds two tensors, so in integers both operands sit on one grid, and with the step set by the larger operand the smaller keeps few levels: in the mix the weights $p\le1$ share the grid of values several times larger. The scaling $a\to\lambda a$, $b\to b/\lambda$ leaves the product unchanged and gives both operands the same range, so one grid fits both. It also makes the squares as small as they can be, since $\lambda^2a^2+b^2/\lambda^2\ge2|ab|$ with equality at $\lambda^2=|b|/|a|$, and the size of the squares sets the step of the node that subtracts them. PTQ calibration already measures the ranges, so $\lambda^2=\max|b|/\max|a|$ follows from the same statistics, and $\lambda$ folds into the fixed affine maps that produce $a$ and $b$ and into the output maps, so it costs nothing at run time. A node that combines several inputs, as the norm expansion does, quantizes each input on its own grid and rescales it before the sum (Methods); this takes care of the subtraction, but the balance of the operands inside $a\pm b$ still sets the number of bits. The strict graph in this paper applies no $\lambda$, and its fixed affine nodes take 12-bit input codes. Brent's scaling can reduce this bit width.

\begin{table}[h]
\caption{\textbf{bf16 error of the value mix on DeiT-Tiny/Pet}, relative to float64, averaged over the 12 blocks of 16 test images. Partial sums and scaled: emulated kernels with bf16 rounding of every stored value; $\lambda$ is fitted on 8 other images. The cancellation factor is $\sum_kv_k^2/|y|$.}
\label{tab:s-mix}
\centering
\begin{tabular}{r c cccc}
\toprule
& & & \multicolumn{3}{c}{Norm expansion}\\
\cmidrule(lr){4-6}
$k$ & Cancellation & Polarization, per term & partial sums & per term & per term, scaled\\
\midrule
16 & 65 & 4.3\% & 9.1\% & 5.0\% & 1.0\%\\
197 & 543 & 14.1\% & 57.4\% & 14.4\% & 2.7\%\\
\bottomrule
\end{tabular}
\end{table}

\section{S6. Tail-mass ablation on DeiT-Tiny}

Table~\ref{tab:s-tail} compares the tail modes on Pet before training, and after Jeffreys distillation with trainable $q$. A fixed small tail mass ($q=10^{-3}$) behaves like tail drop: the omitted mass must be estimated. Table~\ref{tab:s-q} compares the calibrated $q$, a constant per block averaged on the training split, with the exact per-query tail mass on the held-out Pet test split, which no calibration or selection reads. Both sample 4096 queries per block from each minibatch of 32 images. The two means agree within $0.012$, and within $0.001$ for $k\ge16$, while single queries spread widely around them.

This spread leaves room between the two modes. Calibrated mode uses $q=\E\tau$ and exact mode $q=\tau$. Between them lies the conditional mean $\E[\tau\mid\phi]$ for any statistic $\phi$ of what top-$k$ attention reads, such as the log-sum-exp of the kept scores, the $k$-th score and the number $m$ of dropped keys. Its expected squared error is never larger than that of the constant, since $\Var\tau=\E\bigl(\E[\tau\mid\phi]-\tau\bigr)^2+\Var\E[\tau\mid\phi]$. Predicting the dropped partition sum as $m\,g(\phi)$, instead of $\tau$ itself, gives $q=0$ when no key is dropped. It therefore also covers causal attention, where the tail mass grows with the position of the query and one constant per block fits poorly: on Pythia-70M at $k=16$ the exact tail mass on 4 training minibatches is $0.062\pm0.115$, between $0$ and $0.70$.

\begin{table}[h]
\caption{\textbf{Tail modes on DeiT-Tiny/Pet}, top-1 accuracy (\%). The last column is the saved one of five Jeffreys-distilled runs with trainable $q$, the run with the highest training accuracy.}
\label{tab:s-tail}
\centering
\begin{tabular}{r ccccc}
\toprule
$k$ & tail drop & fixed $q=10^{-3}$ & calibrated & exact & calibrated, distilled\\
\midrule
1  & 2.53 & 2.78 & 15.73 & 8.69 & 41.89\\
2  & 3.16 & 2.86 & 42.95 & 21.31 & 64.57\\
4  & 2.83 & 3.11 & 67.10 & 58.60 & 78.09\\
8  & 5.70 & 5.51 & 79.50 & 81.77 & 84.71\\
16 & 39.44 & 39.33 & 86.24 & 88.80 & 88.58\\
32 & 81.71 & 82.07 & 89.23 & 90.84 & 90.19\\
64 & 89.10 & 88.91 & 90.35 & 91.36 & 90.76\\
\bottomrule
\end{tabular}
\end{table}

\begin{table}[h]
\caption{\textbf{Calibrated against exact tail mass on DeiT-Tiny/Pet.} Calibrated: the per-block constant from the training split, averaged over blocks. Exact: per-query tail mass on the held-out Pet test split, mean $\pm$ s.d.\ within a block averaged over blocks (min, max).}
\label{tab:s-q}
\centering
\begin{tabular}{r cc}
\toprule
$k$ & calibrated mean $q$, training & exact $\tau$, test\\
\midrule
1   & 0.7825 & $0.7943\pm0.1540$ (0.0007, 0.9925)\\
2   & 0.7139 & $0.7229\pm0.1703$ (0.0004, 0.9852)\\
4   & 0.6294 & $0.6349\pm0.1839$ (0.0004, 0.9710)\\
8   & 0.5278 & $0.5300\pm0.1885$ (0.0001, 0.9438)\\
16  & 0.4116 & $0.4125\pm0.1770$ (0.0000, 0.8915)\\
32  & 0.2826 & $0.2831\pm0.1475$ (0, 0.7920)\\
64  & 0.1511 & $0.1513\pm0.0980$ (0, 0.6052)\\
96  & 0.0811 & $0.0814\pm0.0607$ (0, 0.4286)\\
128 & 0.0394 & $0.0398\pm0.0330$ (0, 0.2627)\\
197 & 0 & 0\\
\bottomrule
\end{tabular}
\end{table}

\section{S7. Distillation loss}

Table~\ref{tab:s-distill} compares cross-entropy and Jeffreys distillation on Pet. For $2\le k\le32$, Jeffreys distillation recovers $45\%$ to $46\%$ of the accuracy gap between the compiled model and the reference, and cross-entropy $36\%$ to $43\%$. Jeffreys gains most, $1.2$ to $1.5$ points, at $k=2$ and $4$, since it matches the full teacher distribution and not only the label.

\begin{table}[h]
\caption{\textbf{Distillation on DeiT-Tiny/Pet}, top-1 accuracy (\%) of the saved run, the one with the highest training accuracy, and mean $\pm$ s.d. over five seeds.}
\label{tab:s-distill}
\centering
\begin{tabular}{r c cc cc}
\toprule
$k$ & compiled & CE saved & CE mean & Jeffreys saved & Jeffreys mean\\
\midrule
1   & 15.73 & 41.24 & $40.94\pm0.24$ & 41.89 & $41.22\pm0.60$\\
2   & 42.95 & 64.49 & $63.82\pm0.93$ & 64.57 & $65.30\pm0.47$\\
4   & 67.10 & 77.02 & $77.03\pm0.18$ & 78.09 & $78.22\pm0.18$\\
8   & 79.50 & 84.33 & $84.40\pm0.18$ & 84.71 & $84.91\pm0.15$\\
16  & 86.24 & 88.17 & $88.26\pm0.06$ & 88.58 & $88.56\pm0.19$\\
32  & 89.23 & 89.92 & $90.00\pm0.19$ & 90.19 & $90.18\pm0.08$\\
64  & 90.35 & 90.84 & $90.45\pm0.22$ & 90.76 & $90.43\pm0.18$\\
96  & 90.76 & 90.90 & $90.83\pm0.09$ & 91.25 & $90.86\pm0.22$\\
128 & 91.11 & 91.11 & $91.11\pm0.00$ & 91.11 & $91.11\pm0.00$\\
197 & 91.36 & 91.36 & $91.36\pm0.00$ & 91.36 & $91.36\pm0.00$\\
\bottomrule
\end{tabular}
\end{table}

\section{S8. bf16 against fp32}

Table~\ref{tab:s-dtype} compares bf16 with fp32 for every network with all keys. The bf16 column evaluates the saved models: weights in bf16 and arithmetic under CUDA autocast, as in the main tables. The fp32 column uses the fp32 reference weights and IEEE fp32 arithmetic. Each compiled graph is rebuilt in fp32, so its fixed coefficients are fp32 too. With all keys, no calibrated quantity is active. The DeiT-Tiny and Pythia-70M results repeat exactly between runs. MambaIRv2-Light uses nondeterministic kernels, and two runs differ by up to $0.01$\,dB.

In fp32, every compiled graph with exact GELU reproduces its reference (Theorems~\ref{thm:strict} and~\ref{thm:ext}, Proposition~\ref{prop:hsr}). On Pet, the extended, strict and Hadamard graphs match the accuracy and loss of DeiT-Tiny to four digits. The extended graph matches on ImageNet-1k and WikiText-2, and on Set5 within the run-to-run variation. With $x\,\sigma(1.702x)$, the strict and Hadamard graphs compute the same model and give the same fp32 accuracy. bf16 changes the accuracy of DeiT-Tiny by at most $0.46$ points and the PSNR of MambaIRv2-Light by at most $0.04$\,dB.

Pythia-70M runs in fp32. In bf16 its validation NLL rises from $4.55$ to $6.68$ nats, for the reference and the compiled graph alike. Only bf16 weights with fp32 arithmetic give $4.56$. The degradation cause is the logits. They share a large offset that softmax ignores: their median is $1061$, while the top two differ by a median of $0.73$. Near $1000$, bf16 resolves steps of $8$, so it loses the ranking of the tokens. The output projection causes $1.7$ of the $2.1$ nats and attention the other $0.4$.

\begin{table}[h]
\caption{\textbf{bf16 against fp32 with all keys.} Top-1 accuracy (\%) for DeiT-Tiny, PSNR (dB) on the Y channel of Set5 for MambaIRv2-Light, validation NLL (nats) for Pythia-70M. Sigmoid GELU: $x\,\sigma(1.702x)$.}
\label{tab:s-dtype}
\centering
\begin{tabular}{llrrr}
\toprule
Network & Model & bf16 & fp32 & bf16 $-$ fp32\\
\midrule
DeiT-Tiny, Pet & reference & 91.36 & 91.33 & $+0.03$\\
 & extended & 91.36 & 91.33 & $+0.03$\\
 & strict, GELU & 91.06 & 91.33 & $-0.27$\\
 & Hadamard, GELU & 91.28 & 91.33 & $-0.05$\\
 & strict, sigmoid GELU & 90.73 & 90.27 & $+0.46$\\
 & Hadamard, sigmoid GELU & 90.35 & 90.27 & $+0.08$\\
\midrule
DeiT-Tiny, ImageNet-1k & reference & 72.02 & 72.02 & $+0.01$\\
 & extended & 72.01 & 72.02 & $0.00$\\
\midrule
MambaIRv2-Light $\times2$, Set5 & reference & 38.234 & 38.272 & $-0.038$\\
 & extended & 38.231 & 38.270 & $-0.039$\\
\midrule
MambaIRv2-Light $\times4$, Set5 & reference & 32.511 & 32.512 & $-0.002$\\
 & extended & 32.502 & 32.507 & $-0.006$\\
\midrule
Pythia-70M, WikiText-2 & reference & 6.681 & 4.550 & $+2.131$\\
 & extended & 6.680 & 4.550 & $+2.130$\\
\bottomrule
\end{tabular}
\end{table}

\section{S9. Quantization with exact GELU}

Table~\ref{tab:ptq} reports each graph in its deployment configuration: the strict graph with $x\,\sigma(1.702x)$ and 12-bit inputs of its fixed affine nodes, before distillation, and the extended graph with GELU and 8-bit codes, after Jeffreys distillation. Table~\ref{tab:s-ptq-gelu} compares the strict graph with exact GELU and with $x\,\sigma(1.702x)$ under the same PTQ, both before distillation. The quantized graph has 344 nodes with GELU and 380 with its approximation, 332 and 368 at $k=197$. Exact GELU matches no NPU configuration and serves as a comparison. With either GELU, PTQ costs the strict graph at most $1.5$ points, against $2.5$ to $3.9$ points for the extended graph.

\begin{table}[h]
\caption{\textbf{PTQ of the strict graph with exact GELU and with $x\,\sigma(1.702x)$} on DeiT-Tiny/Pet, before distillation, with the PTQ of Table~\ref{tab:ptq}. bf16: top-1 accuracy (\%) of the float model. $\Delta$: change after PTQ with per-channel or per-tensor weight scales. The approximation columns repeat Table~\ref{tab:ptq}.}
\label{tab:s-ptq-gelu}
\centering
\begin{tabular}{r rrr rrr}
\toprule
& \multicolumn{3}{c}{GELU} & \multicolumn{3}{c}{$x\,\sigma(1.702x)$}\\
\cmidrule(lr){2-4}\cmidrule(lr){5-7}
$k$ & bf16 & $\Delta$, per-channel & $\Delta$, per-tensor & bf16 & $\Delta$, per-channel & $\Delta$, per-tensor\\
\midrule
16 & 86.37 & $-1.44$ & $-0.90$ & 84.08 & $-1.47$ & $-1.06$\\
32 & 89.07 & $-0.76$ & $-0.68$ & 87.98 & $-0.68$ & $-0.71$\\
64 & 90.35 & $-0.49$ & $+0.05$ & 89.78 & $-0.05$ & $-0.14$\\
197 & 91.06 & $-0.08$ & $+0.19$ & 90.73 & $-0.46$ & $-0.98$\\
\bottomrule
\end{tabular}
\end{table}

\section{S10. Tensor sizes}

Native $QK^\top$ reads $2Nd$ values per head and writes $N^2$. The unfused strict graph also writes the $N^2d$ sums $q_{if}+k_{jf}$, squares them in place, and reduces over $f$. For DeiT-Tiny that is $2.48\times10^{6}$ values per head per image against $3.88\times10^{4}$, a factor $d=64$, which drives the memory column of Table~\ref{tab:strict-fast}. Top-$k$ mixing reads $k$ of the $N$ value rows per query, a fraction $k/N$ of the value traffic of dense mixing.

\end{document}